\DocumentMetadata{lang=en}
\documentclass{article}

\usepackage[nonatbib,final]{neurips_2025}

\usepackage[utf8]{inputenc} %
\usepackage[T1]{fontenc}    %
\usepackage{hyperref}       %
\usepackage{url}            %
\usepackage{booktabs}       %
\usepackage{amsfonts}       %
\usepackage{nicefrac}       %
\usepackage{microtype}      %
\usepackage{xcolor}         %
\usepackage{placeins}
\usepackage[
  natbib=true,%
  backend=biber,%
  style=numeric-comp,%
  sorting=none,%
  isbn=false,%
  url=false,%
  ]{biblatex}
\DeclareSourcemap{
  \maps[datatype=bibtex, overwrite]{
    \map{
      \step[fieldset=url, null]
      \step[fieldset=editor, null]
      \step[fieldset=location, null]
      \step[fieldset=month, null]
      \step[fieldset=address, null]
      \step[fieldset=series, null]
      \step[fieldset=publisher, null]
    }
    \map{
      \step[fieldsource=doi, final]
      \step[fieldset=eprint, null]
    }  
  }
}

\usepackage{algorithm}
\usepackage{algpseudocode}
\usepackage{wrapfig}

\usepackage{amsmath}
\usepackage{amssymb}
\usepackage{mathtools}
\usepackage{amsthm}

\usepackage{zref-clever}
\zcsetup{cap,nameinlink=false}
\usepackage{thmtools, thm-restate}

\theoremstyle{plain}
\declaretheorem[numberwithin=section]{theorem}
\declaretheorem[sibling=theorem]{proposition}
\declaretheorem[sibling=theorem]{lemma}
\declaretheorem[sibling=theorem]{corollary}
\theoremstyle{definition}
\declaretheorem[sibling=theorem]{definition}

\theoremstyle{remark}

\usepackage{xparse}
\usepackage{nicefrac}       %
\mathtoolsset{showonlyrefs}
\usepackage{bm}                                %

\usepackage{stmaryrd}                          %

\usepackage[normalem]{ulem}                     %
\hypersetup{colorlinks}

\renewcommand{\Pr}{\operatorname*{\mathbb{P}}}
\newcommand*{\mypara}[1]{\leavevmode{\normalfont\normalsize\bfseries #1.}}

\newcommand*{\euler}{\ensuremath{\mathsf{e}}}
\newcommand*{\reals}{\mathbb{R}}
\newcommand*{\naturals}{\mathbb{N}}
\renewcommand{\vec}{\bm}

\newcommand*{\scrD}{\mathcal{D}}

\newcommand*{\scrN}{\mathcal{N}}
\newcommand*{\scrO}{\mathcal{O}}
\newcommand*{\scrP}{\mathcal{P}}
\newcommand*{\scrR}{\mathcal{R}}
\newcommand*{\scrS}{\mathcal{S}}
\newcommand*{\scrT}{\mathcal{T}}
\newcommand*{\scrQ}{\mathcal{Q}}
\newcommand*{\scrX}{\mathcal{X}}
\newcommand*{\scrY}{\ensuremath{\mathcal{Y}}}
\newcommand*{\scrZ}{\ensuremath{\mathcal{Z}}}
\newcommand*{\mech}{\ensuremath{{\mathsf{M}}}}
\newcommand*{\analyst}{\ensuremath{{\mathsf{A}}}}
\newcommand*{\ind}[1]{\mathbf{1}_{\left[#1\right]}}
\newcommand*{\concat}{{\mathbin{+\mkern-10mu+}}}
\newcommand*{\blank}{{{}\cdot{}}}
\newcommand*{\renyi}[3]{\operatorname{D}_{#1} \mathopen{} \left( #2 \middle\| #3 \right)}
\newcommand*{\privloss}[2]{\operatorname{PrivLoss} \mathopen{} \left( #1 \middle\| #2 \right)}
\NewDocumentCommand{\irange}{mg}{\ensuremath{\llbracket #1 \IfNoValueF{#2}{, #2} \rrbracket}}

\DeclareMathOperator{\sinteract}{\ensuremath{\mathsf{I}}}
\DeclareMathOperator{\replace}{rep}
\DeclareMathOperator{\clamp}{clip}

\DeclareMathOperator*{\E}{\mathbb{E}}
\DeclareMathOperator{\unif}{\mathcal{U}}
\DeclareMathOperator{\normal}{\mathcal{N}}
\DeclarePairedDelimiter{\abs}{\lvert}{\rvert}
\DeclarePairedDelimiter{\floor}{\lfloor}{\rfloor}
\DeclareMathOperator{\erfc}{\operatorname{erfc}}
\DeclareMathOperator*{\argmax}{arg\,max}
\DeclareMathOperator*{\argmin}{arg\,min}

\usepackage{tikz}
\usetikzlibrary{arrows, positioning, fit}

\title{Adaptive Data Analysis for Growing Data}

\author{%
    Neil G.\ Marchant%
    \\
  School of Computing \& Information Systems\\
  University of Melbourne, Australia \\
  \texttt{nmarchant@unimelb.edu.au} \\
  \And
  Benjamin I.\ P.\ Rubinstein \\
  School of Computing \& Information Systems\\
  University of Melbourne, Australia \\
  \texttt{brubinstein@unimelb.edu.au} \\
}

\begin{document}

\maketitle

\begin{abstract}
    Reuse of data in adaptive workflows poses challenges regarding overfitting and the statistical validity of results. 
    Previous work has demonstrated that interacting with data via differentially private algorithms can mitigate 
    overfitting, achieving worst-case generalization guarantees with asymptotically optimal data requirements.
    However, such past work assumes data is \emph{static} and cannot accommodate situations where data \emph{grows} over time. 
    In this paper we address this gap, presenting the first generalization bounds for adaptive analysis on dynamic data. 
    We allow the analyst to adaptively \emph{schedule} their queries conditioned on the current size of the data, in 
    addition to previous queries and responses. 
    We also incorporate time-varying empirical accuracy bounds and mechanisms, allowing for tighter 
    guarantees as data accumulates. 
    In a batched query setting, the asymptotic data requirements of our bound grows with the square-root of the number 
    of adaptive queries%
    , matching prior works' improvement over data splitting for the static setting. 
    We instantiate our bound for statistical queries with the clipped Gaussian mechanism, where it empirically 
    outperforms baselines composed from static bounds. 
\end{abstract}

\section{Introduction} \label{sec:intro}

The ubiquity of adaptive workflows in modern data science has raised concerns about the risk of overfitting and the 
validity of findings~\citep{ioannidis2005why,gelman2014statistical}. %
In such adaptive workflows, data is reused over multiple steps, where the procedure or analysis at any given step may 
depend on results of previous steps. 
Common examples include hyperparameter tuning or model selection on a hold-out set~\citep{reunanen2003overfitting,
rao2008dangers,dwork2015generalization}, 
blending of exploratory and confirmatory data analysis~\citep{zhao2017controlling,fife2021understanding}, and the 
reuse of benchmark\slash public datasets within a research community~\citep{koch2021reduced}. 
While adaptivity can enable more exploratory analysis, it is not covered by conventional guarantees of generalization 
and statistical validity, which assume the analysis is selected independently of the data~\citep{dwork2015preserving}.

A simple approach for enabling adaptive data analysis with generalization guarantees is to collect a fresh 
dataset %
whenever a step in the analysis depends on existing data. 
This can also be achieved by randomly splitting a dataset and using a separate split for each step.
However, the data requirements of this approach may be prohibitive, scaling \emph{linearly} in the number of  
adaptive steps. 
A line of work based on algorithmic stability~\citep{dwork2015preserving,bassily2016algorithmic,
feldman2018calibrating,shenfeld2019necessary,fish2020sampling,rogers2020guaranteed,jung2020new,dinur2023differential,
blanc2023subsampling,shenfeld2023generalization} offers a significant improvement over data splitting, with data 
requirements that grow asymptotically with the \emph{square-root} of the number of adaptive steps. 
The core result of this line of work is a \emph{transfer theorem}, which guarantees that the outputs of an adaptive 
analysis are close to the expected outputs on the data distribution if (i)~the analysis is stable under small 
changes to the dataset and (ii)~the outputs are close to the empirical average on the dataset. 
\emph{Differential privacy}~\citep{dwork2006calibrating} is commonly adopted as a  notion of stability in this work, 
achieving the square-root dependence mentioned above, which is asymptotically optimal in the worst 
case~\citep{hardt2014preventing,steinke2015interactive}.

Most prior work on adaptive data analysis via algorithmic stability assumes a common setting, where the data is 
sampled i.i.d.\ from an unknown distribution and used by a mechanism to estimate an analyst's adaptive 
queries~\citep{dwork2015preserving,bassily2016algorithmic,jung2020new}. 
Generalization bounds are then obtained for a worst-case data distribution and a worst-case analyst, who is actively 
trying to overfit.
More recently, variations of this setting have been studied in an attempt to better reflect how data is used in 
practice~\citep{zrnic2019natural,rogers2020guaranteed,kontorovich2022adaptive,shenfeld2023generalization}. 
This includes replacing the assumption of i.i.d.\ data with weakly correlated data~\citep{kontorovich2022adaptive}, 
or replacing a worst-case analyst by a dynamic model~\citep{zrnic2019natural}.
Concurrent work has also made progress on the former, providing generalization guarantees for correlated data by 
constraining the analyst to a class of ``concentrated'' queries~\citep{rapoport2025tight}.

A limitation of all prior work is the assumption that data is collected before the analysis begins, 
and remains \emph{static} thereafter. 
However, it is common in practice for data to \emph{grow} over time, and it may be undesirable to wait for all data to 
arrive or to ignore data that arrives after an analysis has begun. 
This \emph{growing data} setting has been studied in the differential privacy literature, both for adaptive 
queries~\citep{cummings2018differential,qiu2022differential} and for updating a fixed query whenever a 
dataset changes~\citep{dwork2010differential,chan2011private,henzinger2023almost,fichtenberger2023constant}.
In this paper, we bridge the gap, obtaining the first generalization guarantees for adaptive data analysis (ADA) in 
the growing setting. 
We consider a fully adaptive analyst, who can determine not only the content of their queries, but also the timing and 
frequency of their submissions on-the-fly. 
This schedule can be conditioned on the current size of the data, as well as all past queries and responses.

To tackle the growing data setting, we introduce definitions and techniques that extend beyond existing ADA frameworks. 
A key innovation is our approach to bounding query error, which, following \citet{jung2020new}, incorporates a term 
comparing the query results evaluated on a posterior data distribution and the true data distribution. 
Our insight is that for the growing data setting, the posterior distribution must be marginalized over unseen future 
data at the time a query is submitted—a crucial departure from the static setting where the full dataset is known in 
advance.
This yields a transfer theorem that depends on a corresponding variant of \emph{posterior stability}. 
This dynamic nature of the posterior, whose support grows in a way that depends on the analyst's adaptive schedule, 
requires significant new analytical ideas to prove the conversion from differential privacy to posterior 
stability, which are instrumental in obtaining DP-based transfer theorems. 
This results in an additional factor in the error bound (compared to the static case for linear 
queries) proportional to the percentage increase in the dataset size.

We propose a non-uniform generalization of $(\epsilon, \delta)$-differential privacy where the $\delta$ parameter 
varies for each data point\slash time step, inspired by personalized privacy~\citep{jorgensen2015conservative}.
This permits us to obtain tighter generalization guarantees---the error bound increases as a function of the average 
$\delta$ over all time steps, rather than the maximum $\delta$ under the standard DP definition. 
These theoretical advances culminate in new bounds for various query types, including statistical queries, 
low-sensitivity queries, and low-sensitivity minimization queries using non-uniform differential privacy as a 
stability measure. 

As a concrete application of our guarantees we consider using the clipped Gaussian mechanism to answer adaptive 
statistical queries. 
To ensure tight privacy accounting when the number of queries at each time step is chosen adaptively, we leverage a 
privacy filter~\citep{whitehouse2023fully} which supports fully adaptive composition. 
Our bound empirically outperforms baselines composed from bounds for static data. 
In a batched query setting, the asymptotic data requirements of our bound grow with the square-root of the number 
of adaptive queries for a fixed accuracy goal (assuming the ratio of final to initial data size is held constant). 
This improvement matches the improvement of bounds for static data~\citep{jung2020new} over the data splitting 
baseline.

\section{Preliminaries} \zlabel{sec:background}

We introduce notation used throughout the paper. 
The sequence of integers from $n_1$ to $n_2$ inclusive is denoted by $\irange{n_1, n_2}$, or $\irange{n_2}$ 
when $n_1 = 1$. 
Given a sequence $\vec{x}$, we refer to the $t$-th element as $x_t$ and the length as $\abs{\vec{x}}$. 
We use $\vec{x}_{\irange{t_1,t_2}}$ to denote the subsequence of $\vec{x}$ containing elements from 
index $t_1$ to $t_2$ inclusive, or $\vec{x}_{\irange{t_2}}$ when $t_1 = 1$. 
We use capital letters for random variables and lower case letters for realizations of a random variable. 
The uniform distribution over a set $\scrS$ is denoted $\unif(\scrS)$ and the normal distribution with mean $\mu$ and 
standard deviation $\sigma$ is denoted $\normal(\mu, \sigma^2)$. 
The product distribution of~$n$ i.i.d.\ random variables drawn from $\scrD$ is denoted $\scrD^n$.

\begin{figure}[t]
    \centering
    \includegraphics[width=\linewidth]{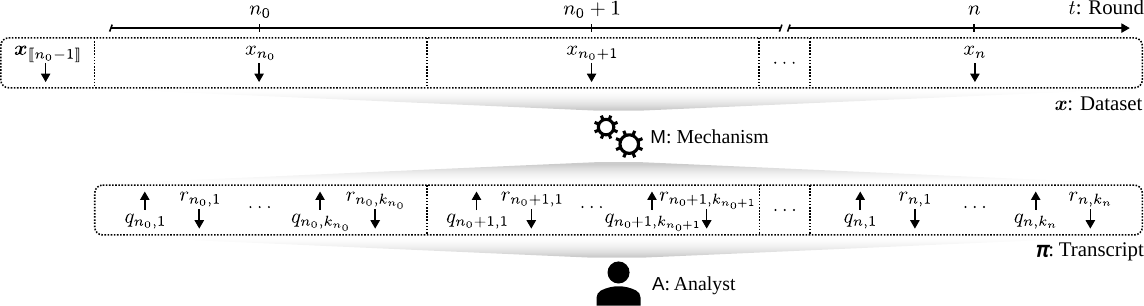}
    \caption{Schematic of our new setting for adaptive data analysis on growing data. 
    The dataset is of size $n_0$ when the analysis begins, and grows by one data point in each round. 
    The analyst asks queries adaptively in each round based on past responses, and receives a response from the 
    mechanism before selecting the next query. 
    The framework reduces to the static data setting when $n = n_0$.
    }
    \label{fig:schematic}
\end{figure}

\subsection{Formulating Adaptive Data Analysis (ADA) for Growing Data} \zlabel{sec:formulate}
We propose a new formulation of adaptive data analysis for growing data that builds on prior work 
for static data~\citep{dwork2015preserving,bassily2016algorithmic,jung2020new}. 

\mypara{Dataset} 
Let $\scrP$ be an unknown data distribution over a finite domain $\scrX$. 
We consider a growing dataset $\vec{X} = (X_1, X_2, \ldots)$, where each data point $X_t$ is drawn i.i.d.\ from 
$\scrP$, and the data points are indexed in order of arrival. 
We define the \emph{snapshot} of $\vec{X}$ at index $t$, to be the portion of the data realized by index $t$, namely 
$\vec{X}_{\irange{t}}$.
We study datasets over a fixed horizon $n$ so that $\abs{\vec{X}} = n$.

\mypara{Analyst and mechanism} 
We consider an analyst~$\analyst$ who would like to estimate queries about the data distribution~$\scrP$ 
\emph{online} using the growing dataset~$\vec{X}$. 
The analyst asks queries from a fixed query class~$\scrQ$, such as the class of \emph{statistical queries} 
(see \zcref{sec:queries}).
The analyst is prohibited from accessing~$\vec{X}$ directly, but can instead submit queries to an online 
mechanism~$\mech$ that returns estimates using the current snapshot of~$\vec{X}$. 
We assume $\mech$ produces \emph{feasible estimates}, meaning estimates are guaranteed to fall within the range 
of the query. 
This can be achieved by design~\citep{hu2024privacy} or by modifying $\mech$ to project infeasible estimates onto the 
range of the query. 

\begin{wrapfigure}{r}{0.5\textwidth}
  \vspace{-24pt}
  \begin{minipage}{\linewidth}
    \begin{algorithm}[H]
    \caption{%
    Interaction between $\analyst$ and $\mech$}
    \zlabel{alg:interaction}
    \small
    \algrenewcommand\algorithmicindent{1.0em}%
    \begin{algorithmic}[1]
      \State Wait for $\mech$ to receive data $\vec{X}_{\irange{n_0 - 1}}$ 
      \label{alg-line:init-data-pts}
      \State Initialize empty transcript $\Pi$ \label{alg-line:init-transcript}
      \For{Round $t \in \irange{n_0}{n}$} \label{alg-line:start-for-loop}
        \State Wait for $\mech$ to receive next data point $X_t$ \label{alg-line:next-data-pt}
        \While{Analyst not finished} \label{alg-line:exit-condition}
          \State Generate query: $Q \sim \analyst(\Pi)$
          \State Estimate query response: $R \sim \mech(Q; \vec{X}_{\irange{t}})$ 
          \State Append $(Q, R)$ to $\Pi_t$ in-place \label{alg-line:update-transcript}
        \EndWhile 
      \EndFor \label{alg-line:end-for-loop}
      \State \Return Transcript $\Pi$
    \end{algorithmic}
    \end{algorithm}
  \end{minipage}
\end{wrapfigure}

\mypara{Interaction} 
\zcref{alg:interaction} specifies how the analyst interacts with the mechanism.
To begin, the analyst selects an initial dataset size $n_0 \leq n$ and waits for the mechanism to receive~$n_0$ data 
points (lines~\ref{alg-line:init-data-pts} and~\ref{alg-line:next-data-pt}).
The analyst then submits queries to the mechanism over multiple rounds, where each round is marked by the receipt 
of a new data point (lines~\ref{alg-line:start-for-loop}--\ref{alg-line:end-for-loop}). 
The rounds are indexed starting at~$n_0$ so that index~$t$ coincides with the size of the growing dataset. 
The number of queries asked in a given round $t$ is determined adaptively by the analyst 
(line~\ref{alg-line:exit-condition}). 
Since the analyst adaptively controls when each round terminates, they can force the mechanism to use a stale 
snapshot of the dataset $\vec{X}_{\irange{t}}$ even if 
newer data points (with indices $>t$) have arrived and are waiting to be ingested by the mechanism. 

\mypara{Transcript}
The interaction yields a transcript $\Pi$ of the queries submitted in each round and the estimates produced by the 
mechanism (lines~\ref{alg-line:init-transcript} and~\ref{alg-line:update-transcript}). 
The transcript is structured as a sequence of sequences $\Pi = (\Pi_1, \Pi_2, \ldots, \Pi_n)$ where 
$\Pi_t = ((Q_{t,1}, R_{t,1}), \ldots, (Q_{t,k_t}, R_{t,k_t}))$ records the query-estimate pairs from round $t$ 
in the order they were submitted. 
We denote the space of possible transcripts by 
$\scrT = \bigcup_{\vec{k} \in S(n, k)} \prod_{t = 1}^{n} (\scrQ \times \scrR)^{k_t}$, 
where $\scrQ$ is the query class, $\scrR$ is the range of the queries and 
$S(n, k) = \{ \vec{k} \in \irange{k}^n : (\forall t < n_0)(k_t = 0) \wedge \sum_{t = 1}^{n} k_t = k \}$ 
is the set of possible allocations of $k$ queries across $n$ rounds.\footnote{%
  In the definition of $S(n, k)$, the first statement in the predicate accounts for the fact that the analyst does 
  not begin submitting queries until round~$n_0$.
} 

Looking ahead, we will be interested in the stability of the transcript under perturbations to the dataset. 
It is therefore convenient to interpret the interaction in \zcref{alg:interaction} as a random map 
$\sinteract(\vec{X}; \analyst, \mech)$ that takes a growing dataset $\vec{X} \in \scrX^n$ as input and returns a 
transcript $\Pi \in \scrT$ as output. 
We view $\analyst$ and $\mech$ as parameters of $\sinteract$, and drop the dependence on them where it is clear 
from context.

\subsection{Query Classes} \zlabel{sec:queries}
Following prior work~\citep{bassily2016algorithmic,jung2020new}, we consider three query classes. 
We use $q(\vec{x}_{\irange{t}})$ to denote the result of a query $q \in \scrQ$ evaluated on a 
snapshot $\vec{x}_{\irange{t}}$
and $q(\scrD)$ to denote the result evaluated on a data distribution $\scrD$. 

\noindent\textbf{Low-sensitivity queries} 
are defined by a $\vec{\Delta}$-sensitive function $q: \scrX^* \to [0, 1]$ that maps a data 
snapshot to a scalar on the unit interval.
We say $q$ is $\vec{\Delta}$-sensitive given $\vec{\Delta} = (\Delta_1, \ldots, \Delta_n) \in \reals_+^n$, if for all 
$t \in \irange{n}$ 
we have $\abs{q(\vec{x}_{\irange{t}}) - q(\tilde{\vec{x}}_{\irange{t}})} \leq \Delta_t$ for any pair of neighboring 
snapshots $\vec{x}_{\irange{t}}, \tilde{\vec{x}}_{\irange{t}} \in \scrX^t$ that differ on one data point. 
The result of the query when evaluated on a data distribution $\scrD$ is 
$q(\scrD) \coloneqq \E_{\vec{X} \sim \scrD} [q(\vec{X})]$.

\noindent\textbf{Statistical queries} 
are a subset of $\vec{\Delta}$-sensitive queries where each query is of the form 
$q(\vec{x}_{\irange{t}}) = \sum_{\tau = 1}^{t} \tilde{q}(x_\tau) / t$ for some function $\tilde{q}: \scrX \to [0, 1]$. 
Since each query is fully specified by $\tilde{q}$, we refer to $\tilde{q}$ as $q$ when there is no ambiguity.
The sensitivity satisfies $\Delta_t \leq 1/t$ for all $t \in \irange{n}$.

\noindent\textbf{Minimization queries}
are solutions to parameter optimization problems defined by a data-dependent loss function. 
Due to space constraints, we discuss them in \zcref{app:gen-stability-min-queries}.

\subsection{Generalization and Stability} \zlabel{sec:gen-stability}

\zcref{alg:interaction} may fail to generalize if the mechanism leaks detailed information 
about the dataset that is exploited by the analyst when selecting queries. 
The degree of leakage is related to the stability of the interaction under perturbations to the dataset. 
Roughly speaking, a more stable interaction leaks less information and is more likely to generalize.
In \zcref{sec:theory}, we will derive generalization guarantees that depend on the stability of the interaction. 
In preparation, we now define how generalization and stability will be measured.
For clarity of exposition, we focus on low-sensitivity and statistical queries here, and extend to 
minimization queries in \zcref{app:gen-stability-min-queries}.

Consider the mechanism's response $R$ to query $Q$ in round $t$, for which the ``true'' answer to the query is the 
expected value on the data distribution, denoted $Q(\scrP^t)$. 
We measure generalization of $R$ in terms of the absolute difference $\abs{R - Q(\scrP^t)}$, which we
refer to as the \emph{distributional error}. 
Our generalization guarantee for the analysis as a whole, takes the form of a high probability bound on the worst-case 
distributional error that holds jointly over all rounds, as defined below.
Note that we consider bounds on the error $\alpha_t$ that vary as a function of the round index $t$, which permits the 
bound to improve as the dataset grows.
\begin{definition} \zlabel{def:distributional-acc}
  Let $\alpha_t \geq 0$ for all $t \in \irange{n_0, n}$ and $\beta \geq 0$. 
  A mechanism $\mech$ is $(\{\alpha_t\}, \beta)$-\emph{distributionally accurate} if with probability $1 - \beta$ 
  over the randomness in the dataset $\vec{X} \sim \scrP^n$ and transcript 
  $\Pi \sim \sinteract(\vec{X}; \analyst, \mech)$, the largest distributional error in the $t$-th 
  round satisfies $\max_{(Q, R) \in \Pi_t} \abs{R - Q(\scrP^t)} \leq \alpha_t$, and this holds jointly for all 
  $t \in \irange{n_0, n}$, for any analyst $\analyst$ and any data distribution $\scrP$.
\end{definition}

When deriving distributional accuracy bounds in the next section, we make use of a related accuracy bound that 
compares the mechanism's responses to raw empirical estimates evaluated on the current data snapshot. 
Consider again the mechanism's response $R$ to query $Q$ in round $t$, where the raw estimate using snapshot 
$\vec{X}_{\irange{t}}$ is denoted $Q(\vec{X}_{\irange{t}})$. 
We define the \emph{snapshot error} of $R$ to be the absolute difference $\abs{R - Q(\vec{X}_{\irange{t}})}$.\footnote{
  $R$ and $Q(\vec{X}_{\irange{t}})$ do not generally coincide, since the mechanism may inject noise in its estimates.
}
By analogy with \zcref{def:distributional-acc}, we then define the following accuracy bound using snapshot error.
\begin{definition} \zlabel{def:snapshot-acc}
  Let $\alpha_t \geq 0$ for all $t \in \irange{n_0, n}$ and $\beta \geq 0$. 
  A mechanism $\mech$ is $(\{\alpha_t\}, \beta)$-\emph{snapshot accurate}\footnote{
    \citet{cummings2018differential} adopt a similar definition of accuracy for a DP mechanism
    operating on a growing dataset. 
    However their definition assumes a non-adaptive analyst and holds for a worst-case dataset, whereas ours holds for 
    a worst-case adaptive analyst assuming the growing dataset is drawn from $\scrP^n$.
  } if with probability $1 - \beta$ over the randomness in the dataset $\vec{X} \sim \scrP^n$ and transcript 
  $\Pi \sim \sinteract(\vec{X}; \analyst, \mech)$, the largest snapshot error in the $t$-th round satisfies 
  $\max_{(Q, R) \in \Pi_t} \abs{R - Q(\vec{X}_{\irange{t}})} \leq \alpha_t$, and this holds for all 
  $t \in \irange{n_0, n}$, for any analyst $\analyst$ and any data distribution $\scrP$.
\end{definition}

As previously mentioned, our generalization guarantees depend on the stability of the interaction.
We adapt the notion of \emph{posterior stability} introduced by \citet{jung2020new} to the growing data setting. 
For a query $Q$ in round~$t$, it measures stability in terms of the absolute difference between the ``true'' answer  
evaluated on the data distribution $\scrP^t$, and the answer evaluated on the posterior data distribution 
$\scrQ_{\Pi}^t \coloneqq \scrP^t \mid \Pi$. 
As in the static setting, the posterior data distribution is conditioned on the full transcript $\Pi$ at the end of 
the interaction, however unlike the static setting, the distribution is only taken over the data available up to 
round~$t$.
\begin{definition}[\citealp{jung2020new}] \zlabel{def:posterior-stability}
  Let $\epsilon, \delta \geq 0$. 
  An interaction $\sinteract(\blank; \blank, \mech)$ is $(\epsilon, \delta)$-\emph{posterior stable}, or 
  $(\epsilon, \delta)$-PS for short, 
  if with probability $1 - \delta$ over the randomness in the dataset $\vec{X} \sim \scrP^n$ and transcript 
  $\Pi \sim \sinteract(\vec{X}; \analyst, \mech)$, we have 
  $\max_{(Q, R) \in \Pi_t}  \abs{Q(\scrP^t) - Q(\scrQ_\Pi^t)} \leq \epsilon$ for 
  all $t \in \irange{n_0, n}$, and this holds for any analyst $\analyst$ and any data distribution $\scrP$.
\end{definition}

We will show that posterior stability follows from \emph{differential privacy}~\citep{dwork2006calibrating}.
We consider a variation of the standard definition of differential privacy, where the level of privacy\slash stability 
varies non-uniformly over data points. 
Similar definitions have been used in dynamic data settings~\citep{ebadi2015differential,stemmer2024private} to 
facilitate tighter privacy accounting. 
We have the same motivation here---by bounding the $\delta$ parameter for each data point separately, we can 
obtain tighter generalization bounds. 
\begin{definition} \zlabel{def:approx-dp-interaction}
  Let $\epsilon \geq 0$ and $\vec{\delta} \colon \irange{n} \to [0, 1]$. 
  An interaction $\sinteract(\blank; \blank, \mech)$ is $(\epsilon, \vec{\delta})$-\emph{differentially private}, or 
  $(\epsilon, \vec{\delta})$-DP for short, if for all analysts $\analyst$, all rounds $t \in \irange{n}$,
  all pairs of neighboring growing datasets $(\vec{x}, \vec{x}') \in \scrN_t$ differing on the $t$-th data point,
  and all measurable events $E \subseteq \scrT$: 
  \begin{equation}
    \Pr \mathopen{} \left(\sinteract( \vec{x}; \analyst, \mech) \in E\right)
      \leq \euler^{\epsilon} \Pr \left(\sinteract( \vec{x}'; \analyst, \mech) \in E\right) + \vec{\delta}(t).
  \end{equation}
\end{definition}
We consider a \emph{bounded} neighboring relation, meaning that $\scrN_t$ contains pairs of datasets 
$(\vec{x}, \vec{x}')$ where $\vec{x}$ can be obtained from $\vec{x}'$ by replacing the $t$-th data point.
We note that this non-uniform privacy definition includes the standard (uniform) definition as a special case: in 
particular, $(\epsilon, \vec{\delta})$-DP implies $(\epsilon, \max_{t} \vec{\delta}(t))$-DP. 
We refer the reader to \zcref{app:non-uniform-privacy} for discussion and results on non-uniform DP.

As a stability measure for adaptive data analysis, DP has several advantages. 
First, DP can be interpreted as a privacy guarantee, not merely a bound on stability.  
Second, DP has been widely studied in statistics and machine learning for almost two decades, so there are established 
private mechanisms for common data analysis and learning tasks. 
Third, DP supports composition which simplifies privacy accounting of the interaction. 
For example, one could instantiate a privacy filter~\citep{feldman2021individual,whitehouse2023fully} for a 
differentially private query-answering mechanism with target privacy parameters $\epsilon$ and $\delta$. 
This then allows for adaptive selection of the analyst's queries, the mechanism's algorithm, and the privacy parameters 
for individual queries, noting that the interaction may be forced to terminate early to ensure it satisfies 
$(\epsilon, \delta)$-DP. 
The post-processing property of DP is essential for this accounting to work, as the analyst's selection of the next 
query is viewed as post-processing on the mechanism's private estimates to previous queries. 
As a result, the analyst does not accrue an additional cost to privacy\slash stability, even in the worst case. 

\section{Generalization Guarantees for ADA on Growing Data} \zlabel{sec:theory} 

In this section, we present generalization guarantees for ADA in the growing data setting. 
\zcref[S]{fig:theory-outline} provides an outline of our key results, where the generalization guarantees 
(a.k.a.\ transfer theorems) are shaded in blue. 
Due to space constraints, we present selected results here and defer proofs and technical results to 
\zcref{app:proofs-theory}.

\begin{figure}
  \centering
  \resizebox{0.93\linewidth}{!}{%
    \begin{tikzpicture}[node distance=0.5cm and 0.5cm]
      \tikzstyle{def} = [rectangle, 
        rounded corners, 
        minimum width=2cm, 
        minimum height=0.95cm,
        align=center, 
        draw=black, 
        fill=yellow!30
      ]
      \tikzstyle{thm} = [def, fill=cyan!30]
      \tikzstyle{lem} = [thm, fill=orange!30]
      \tikzstyle{query} = [rectangle, 
        align=center
      ]
      \tikzstyle{querybox} = [rectangle, 
        rounded corners, 
        dashed,
        draw=black,
        inner sep=4pt
      ]
      \tikzstyle{arrow} = [thick,-latex]
      \node (def2) [def] {\textbf{\zcref{def:snapshot-acc}}\\Snapshot accuracy};
      \node (lem1) [lem, below=of def2] {\textbf{\zcref{lem:cond-acc}}\\Posterior accuracy};
      \node (lem6) [lem, below=of lem1] {\textbf{\zcref{lem:resampling}}\\Resampling Lemma};
      \node (thm1) [thm, right=of lem1] {\textbf{\zcref{thm:transfer-gen}}\\PS transfer theorem};
      \node (def3) [def, below=of thm1] {\textbf{\zcref{def:posterior-stability}}\\Posterior stability};
      \node (thm3) [thm, right=of thm1, xshift=4pt] {\textbf{\zcref{thm:transfer-low-sens-q}}\\DP transfer theorem};
      \node (lem3) [lem, right=of thm3] {\textbf{\zcref{lem:dp-ps-low-sens}}\\DP to PS conversion};
      \node (thm2) [thm, above=of thm3] {\textbf{\zcref{thm:transfer-dp}}\\DP transfer theorem};
      \node (lem2) [lem, right=of thm2] {\textbf{\zcref{lem:dp-ps-stat}}\\DP to PS conversion};
      \node (thm4) [thm, below=of thm3] {\textbf{\zcref{thm:transfer-min-q}}\\DP transfer theorem};
      \node (def4) [def, right=of lem2] {\textbf{\zcref{def:approx-dp-interaction}}\\Differential privacy};
      \node (statq) [rectangle, align=center, left=of thm2, xshift=0.4cm] {Statistical\\queries};
      \node (lowsensq) [rectangle, align=center, right=of lem3, xshift=-0.4cm] {$\Delta$-sensitive\\queries};
      \node (minq) [rectangle, align=center, right=of thm4, xshift=-0.4cm] {Min.\\queries};
      \node (boundedq) [rectangle, align=center, left=of lem1, xshift=0.4cm] {Bounded\\queries};
      \node (statqbox) [querybox, fit=(thm2) (lem2) (statq)] {};
      \node (lowsensqbox) [querybox, fit=(thm3) (lem3) (lowsensq)] {};
      \node (minqbox) [querybox, fit=(thm4) (minq)] {};
      \node (boundedqbox) [querybox, fit=(lem1) (thm1) (boundedq)] {};
      \draw [arrow] (def2) -- (lem1);
      \draw [arrow] (lem6) -- (lem1);
      \draw [arrow] (def3) -- (thm1);
      \draw [arrow] (lem1) -- (thm1);
      \draw [arrow] (def4) -- (lem2);
      \draw [arrow] (def4) -- (lem3);
      \draw [arrow] (lem2) -- (thm2);
      \draw [arrow] (thm1) -- (thm2);
      \draw [arrow] (lem3) -- (thm3);
      \draw [arrow] (thm1) -- (thm3);
      \draw [arrow] (thm3) -- (thm4);
    \end{tikzpicture}
  }%
  \caption{Outline of results in \zcref{sec:theory}. Arrows indicate key dependencies, and dashed boxes 
    indicate results that hold for a particular class of queries.}
  \zlabel{fig:theory-outline}
\end{figure}

Recall that our aim is to obtain generalization guarantees in the form of high probability bounds on worst-case
distributional accuracy.
Our proof technique is based on \citet{jung2020new}, whose bounds for static data outperform 
prior work~\citep{dwork2015preserving,bassily2016algorithmic}. 
The core idea of the proof involves decomposing the distributional error into two terms using the triangle inequality:
\begin{equation}
  \abs{r - q(\scrP^t)} \leq \abs{r - q(\scrQ_{\pi}^t)} + \abs{q(\scrQ_{\pi}^t) - q(\scrP^t)}.
  \label{eqn:dist-err-tri-inequality}
\end{equation}
Rather than comparing the response $r$ to query $q$ with the true value $q(\scrP^t)$ directly, we instead compare 
$r$ and $q(\scrP^t)$ with an intermediate value $q(\scrQ_{\pi}^t)$, which is the expectation of the query on the 
posterior distribution of the snapshot at round $t$ conditioned on the final transcript $\pi$. 
This intermediate value is chosen to align with the definitions of snapshot accuracy and posterior stability, which we 
have adapted for the growing data setting. 

We obtain worst-case probabilistic bounds on each term in \eqref{eqn:dist-err-tri-inequality} separately: 
a bound on the first term follows indirectly from snapshot accuracy and a bound on the second term follows directly 
from posterior stability. 
The bound we obtain for the first term is stated below.
\begin{restatable}{lemma}{lemcondacc} \zlabel{lem:cond-acc}
  Suppose $\mech$ is $(\{\alpha_t\}, \beta)$-snapshot accurate for $[0,1]$-bounded queries. 
  Then for any $c > 0$, with probability $1 - \frac{\beta}{c}$ with respect to the randomness in the dataset 
  $\vec{X} \sim \scrP^n$ and transcript $\Pi \sim \sinteract(\vec{X}; \analyst, \mech)$, we have for all 
  $t \in \irange{n_0, n}$ that
  $\max_{(Q, R) \in \Pi_t} \abs{R - Q(\scrQ_{\Pi}^t)} \leq \alpha_t + c$.
\end{restatable}
The proof relies on an elementary observation stated in \zcref{lem:resampling} of 
\zcref{app:proofs-theory}: that the joint distribution on datasets and transcripts does not change when the entire 
dataset is resampled from the posterior distribution $\scrQ_{\Pi}^{n}$ in the final round $n$. 
At first glance, this observation may not seem useful, as the expectation of the query in 
\eqref{eqn:dist-err-tri-inequality} is taken with respect to the posterior distribution of the data available at 
round $t$ when the query was submitted $\scrQ_{\Pi}^t$, not the posterior distribution of the entire dataset 
$\scrQ_{\Pi}^{n}$. 
However it turns out this is not a problem:
Since the event of interest does not depend on data points received after the round $t^\star$ when the worst-case 
deviation from the posterior expectation occurs, 
we can apply \zcref{lem:resampling} and marginalize over the remaining data points.

Combining \zcref{lem:cond-acc} with posterior stability yields our first generalization guarantee. 
\begin{restatable}[PS transfer theorem]{theorem}{thmtransfergen} \zlabel{thm:transfer-gen}
  Suppose $\mech$ is an $(\{\alpha_t\}, \beta)$-snapshot accurate mechanism for $[0,1]$-bounded queries and 
  $\sinteract(\blank; \blank, \mech)$ is an $(\epsilon, \delta)$-posterior stable interaction. 
  Then for every $c > 0$, 
  $\mech$ is $(\{\alpha_t'\}, \beta')$-distributionally accurate for $\alpha_t' = \alpha_t + c + \epsilon$ and 
  $\beta' = \frac{\beta}{c} + \delta$.
\end{restatable}
When the error bounds $\{\alpha_t\}$ are constant, we recover the same bound as \citet[Theorem~4]{jung2020new}, albeit 
for the more general growing data setting. 

We now turn to deriving generalization guarantees using differential privacy as a measure of stability. 
These guarantees are derived by first converting differential privacy (DP) to posterior stability (PS) in a way that 
exploits the structure of the query class, and then invoking \zcref{thm:transfer-gen}.
These steps are visualized in \zcref{fig:theory-outline}, where the DP to PS conversion results in 
\zcref{lem:dp-ps-stat,lem:dp-ps-low-sens} lead to generalization guarantees in 
\zcref{thm:transfer-dp,thm:transfer-low-sens-q}.

We begin with a conversion result for statistical queries. 
\begin{restatable}{lemma}{lemdppsstat} \zlabel{lem:dp-ps-stat} 
  An $(\epsilon, \vec{\delta})$-DP interaction $\sinteract(\blank; \blank, \mech)$ for statistical 
  queries is $(\epsilon', \delta')$-PS for $\epsilon' = \euler^{\epsilon} - 1 + 2 c \sum_{t = 1}^{n} \vec{\delta}(t) / n_0$,
  $\delta' = 1 / c$ and any $c > 0$.
\end{restatable}
If $\epsilon$ and $c$ are both constant, then the scaling of the lower bound $\epsilon'$ as a function of the final 
dataset size $n$ depends on the functional form of $\sum_{t = 1}^{n} \vec{\delta}(t)$. 
In the worst case where $\vec{\delta}(t)$ is uniform, we see that $\epsilon'$ grows linearly in $n$.
In the static setting, where $n = n_0$ and $\vec{\delta}(t) = \delta$, this factor disappears and we recover the 
result of \citet[Lemma~7]{jung2020new}. 
Combining this lemma with \zcref{thm:transfer-gen} yields a generalization guarantee for statistical queries.
\begin{theorem} \zlabel{thm:transfer-dp}
  Suppose $\mech$ is an $(\{\alpha_t\}, \beta)$-snapshot accurate mechanism and $\sinteract(\blank; \blank, \mech)$ 
  is an $(\epsilon, \vec{\delta})$-DP interaction for statistical queries.
  Then for any constants $c, d > 0$, $\mech$ is $(\alpha_t', \beta')$-distributionally accurate for 
  $\alpha_t' = \alpha_t + \euler^\epsilon - 1 + 2 c \sum_{t = 1}^{n} \vec{\delta}(t) / n_0 + d$ and 
  $\beta' = \beta/d + 1 / c$.
\end{theorem}

Next we consider low-sensitivity queries, where we obtain the following DP to PS conversion result.
\begin{restatable}{lemma}{lemdppslowsens} \zlabel{lem:dp-ps-low-sens}
  An $(\epsilon, \vec{\delta})$-DP interaction $\sinteract(\blank; \blank, \mech)$ for 
  $\vec{\Delta}$-sensitive queries is $(\epsilon', \delta')$-posterior stable for 
  $\epsilon' = \euler^{\epsilon} \max_{\tau_1 \in \irange{n_0, n}} \tau_1 \Delta_{\tau_1} - \min_{\tau_2 \in 
  \irange{n_0, n}} \tau_2 \Delta_{\tau_2} 
  + 4 c \left(\sum_{t = 1}^{n} \vec{\delta}(t)\right) \max_{\tau_3 \in \irange{n_0, n}} \Delta_{\tau_3}$, 
  $\delta' = 1 / c$ and any $c > 0$.
\end{restatable}
We see that the lower bound on the posterior stability depends on the extreme values of the sensitivity 
and the sensitivity weighted by dataset size. 
It is interesting to apply this lemma to statistical queries, which are a subset of $\vec{\Delta}$-sensitive queries 
with $\Delta_t \leq 1 / t$. 
We find $\epsilon' = \euler^{\epsilon} - 1 + 4 c \sum_{t = 1}^{n} \vec{\delta}(t) / n_0$, which is looser than 
\zcref{lem:dp-ps-stat} by a factor of~2 in the last term.
We again point out that the bound reduces to \citet[Lemma~15]{jung2020new} in the static case, where $n = n_0$ 
and $\vec{\delta}(t) = \delta$. 
By combining this lemma with \zcref{thm:transfer-gen} we obtain a generalization guarantee for low-sensitivity queries.
\begin{theorem} \zlabel{thm:transfer-low-sens-q}
  Suppose $\mech$ is an $(\{\alpha_t\}, \beta)$-snapshot accurate mechanism and $\sinteract(\blank; \blank, \mech)$ 
  is an $(\epsilon, \vec{\delta})$-DP interaction for $\vec{\Delta}$-sensitive queries.
  Then for any constants $c, d > 0$, $\mech$ is $(\{\alpha'_t\}, \beta')$-distributionally 
  accurate for 
  $\alpha'_t = \alpha_t + \euler^{\epsilon} \max_{\tau_1 \in \irange{n_0, n}} \tau_1 \Delta_{\tau_1} 
    - \min_{\tau_2 \in \irange{n_0, n}} \tau_2 \Delta_{\tau_2} 
    + 4 c \left(\sum_{t = 1}^{n} \vec{\delta}(t)\right) \max_{\tau_3 \in \irange{n_0, n}} \Delta_{\tau_3} + d$ 
  and $\beta' = \beta / d + 1 / c$.
\end{theorem}

Finally, we provide a generalization guarantee for minimization queries in \zcref{thm:transfer-min-q} of 
\zcref{app:gen-stability-min-queries}.

\section{Application: Gaussian Mechanism} \zlabel{sec:application}

In this section, we instantiate our generalization guarantees for statistical queries using the Gaussian mechanism. 
We focus on the Gaussian mechanism since it is simple to describe, it is easily adapted for growing data, and it 
is known to be optimal for answering a small number of queries $k \ll n^2$ on a static dataset of size 
$n$~\citep{bassily2016algorithmic}.\footnote{%
  In the regime where $k \gg n^2$, a variant of the private multiplicative weights mechanism for growing 
  data can be used instead~\citep{cummings2018differential}.
}
For ease of exposition, we assume the number of queries asked in each round is fixed before the analysis begins, 
while the queries themselves are adaptively chosen.   
This simplifies the privacy accounting and makes for a more direct comparison with prior work in the static 
setting~\citep{jung2020new}. 
In Appendix~\ref{app:application-fully-adaptive}, we remove this assumption, presenting a guarantee for the more 
general setting where the analyst adaptively decides how many queries to ask in each round. 
This guarantee relies on fully adaptive composition via a privacy filter~\citep{whitehouse2023fully}. 
All proofs for this section can be found in \zcref{app:proofs-application}.

\subsection{Generalization Guarantee} \zlabel{sec:application-guarantee}

We begin by defining the Gaussian mechanism for growing data. 
We also include a clipped variant that produces feasible estimates to $[0,1]$-bounded queries (see 
\zcref{sec:formulate}).

\begin{definition}
  The \emph{Gaussian mechanism} perturbs an estimate to a query $q$ based on snapshot $\vec{x}_{t}$ by adding 
  Gaussian noise with round-dependent standard deviation $\sigma_t > 0$. 
  Specifically, we have $\mech(q; \vec{x}_t) = q(\vec{x}_t) + z$ with $z \sim \normal(0, \sigma_t^2)$. 
  The \emph{clipped Gaussian mechanism} composes the Gaussian mechanism with the function 
  $\clamp_{[0, 1]}(x) = \max(0, \min(x, 1))$ as a post-processing step.
\end{definition}

We now analyze privacy for the clipped Gaussian mechanism. 
The result below is obtained using zero-concentrated differential privacy (zCDP, see \zcref{def:zcdp}), as it provides 
sharp composition bounds for the Gaussian mechanism~\citep{bun2016concentrated}. 
After proving that the interaction satisfies $\vec{\rho}$-zCDP, we convert to $(\epsilon, \vec{\delta})$-DP 
using \zcref{cor:zcdp-approx-dp-conversion}, which generalizes a result of \citet[Corollary~13]{canonne2020discrete} 
to non-uniform privacy parameters. 
\begin{restatable}{lemma}{lemgaussmechdp} \zlabel{lem:gauss-mech-dp}
  Consider an interaction $\sinteract(\cdot; \cdot, \mech)$ where $\mech$ is the ordinary or clipped Gaussian 
  mechanism. 
  Suppose the analyst decides to submit $k_\tau$ statistical queries in round $\tau \in \irange{n_0, n}$ before 
  $\sinteract(\cdot; \cdot, \mech)$ is executed.
  Then $\sinteract(\cdot; \cdot, \mech)$ satisfies $(\epsilon, \vec{\delta})$-DP for any $\epsilon > 0$ 
  and $\vec{\delta}(t) \leq \psi(\gamma^\star, \vec{\rho}(t), \epsilon)$, where
  $\vec{\rho}(t) = \sum_{\tau = n_0}^{n} k_\tau \ind{t \leq \tau} / 2 \sigma_\tau^2 \tau^2$, 
  $\psi(\gamma, \rho, \epsilon) = \euler^{(\gamma - 1)(\gamma \rho - \epsilon)}
  \left(1 - \gamma^{-1}\right)^{\gamma} / (\gamma - 1)$ and 
  $\gamma^\star = \arg \min_{\gamma \in (1, \infty)} \psi \mathopen{} \left(\gamma, \max_{t \in \irange{n}} \vec{\rho}(t), \epsilon\right)$.
\end{restatable}

Next we analyze snapshot accuracy. 
The bound depends on the inverse CDF of the Gaussian distribution, which is related to the inverse
complementary error function $\erfc^{-1}$.
\begin{restatable}{lemma}{lemgaussmechacc} \zlabel{lem:gauss-mech-acc}
  For any $\beta \in (0, 1)$, the clipped Gaussian mechanism with $\sigma_t \propto \alpha_t$ is 
  $(\{\alpha_t\}, \beta)$-snapshot accurate for $k$ queries with 
  \begin{equation}
    \frac{\alpha_t}{\sqrt{2} \sigma_t} = \erfc^{-1} \mathopen{} \left( 2 - 2 \left(1 - \frac{\beta}{2}\right)^{\frac{1}{k}} \right) 
    < \erfc^{-1} \mathopen{} \left( \frac{\beta}{k} \right).
  \end{equation}
\end{restatable}

Combining \zcref{lem:gauss-mech-dp,lem:gauss-mech-acc} with \zcref{thm:transfer-dp} yields the following 
generalization guarantee.
\begin{restatable}{theorem}{thmgausmechdistacc} \zlabel{thm:gaus-mech-dist-acc}
  Suppose the conditions of \zcref{lem:gauss-mech-dp} hold and assume $\sigma_t = \sigma > 0$.
  Then the clipped Gaussian mechanism is $(\alpha', \beta')$-distributionally accurate for $k$ statistical 
  queries for any $\beta' \in (0, 1)$ and 
  $\alpha' = \min_{\sigma, \beta, \epsilon \in \Theta} \lambda(\sigma, \beta, \epsilon)$, where 
  \begin{align}
    \lambda(\sigma, \beta, \epsilon) &= 
      \sqrt{2} \sigma \erfc^{-1} \mathopen{} \left(\frac{\beta}{k}\right) + \euler^{\epsilon} - 1 + \frac{\beta}{\beta'} 
       + \frac{2 \sum_{\tau = 1}^{n} \vec{\delta}(\tau)}{n_0 \beta'} 
        + \frac{2}{\beta'} \sqrt{\frac{2 \beta \sum_{\tau = 1}^{n} \vec{\delta}(\tau)}{n_0}},
  \end{align}
  $\Theta = \{ (\sigma, \beta, \epsilon) \in \reals^3 : \sigma > 0, 0 < \beta < 1, \epsilon \geq 0\}$ and 
  $\vec{\delta}(\cdot)$ is defined in \zcref{lem:gauss-mech-dp}.
\end{restatable}

\subsection{Empirical Comparison with Alternative Guarantees} \zlabel{sec:application-empirical}

\begin{figure*}
  \centering
  \includegraphics[width=\linewidth]{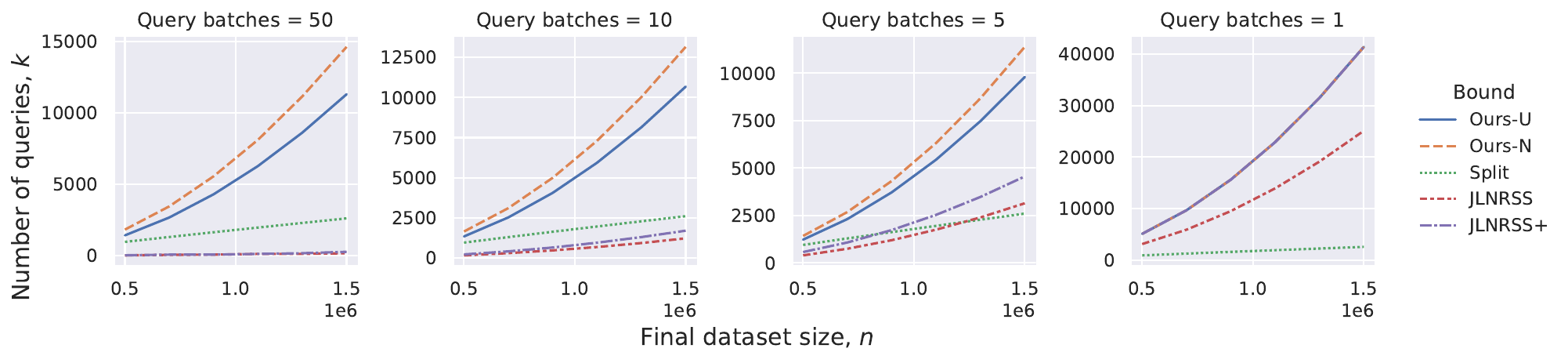}
  \caption{Comparison of the number of adaptive statistical queries that can be answered with error tolerance 
    $\alpha = 0.1$ and uniform coverage probability $1 - \beta = 0.95$ using a growing dataset with growth ratio 
    $n / n_0 = 3$ in a batched query setting. 
    The number of queries (vertical axis) is plotted as a function of the final dataset size $n$ (horizontal axis), 
    bound (curve style) and the number of query batches $b$ (horizontal panel). 
    The right-most panel ($b = 1$), represents the static baseline setting where the analyst forgoes intermediate 
    responses and submits all queries only after the entire dataset of size $n$ has arrived.}
  \zlabel{fig:K-vs-E-alpha-0.1-fixed-growth-ratio}
\end{figure*}

We empirically compare our generalization bounds for growing data with baselines composed from bounds for static 
data. 
When instantiating the bounds, we must specify how many queries $k_t$ are submitted in each round $t$.\footnote{
  Since the bounds are worst case with respect to the analyst and data distribution, no simulation is necessary.
} 
For simplicity, we assume $k$ queries are evenly split into $b$ batches, with a batch being submitted every 
$T = (n - n_0) / b$ rounds.\footnote{%
  When division by $b$ yields a remainder $r$, we distributed $r$ evenly across the first $r$ 
  batches\slash rounds. 
}
Concretely, for $b > 1$ we assume $k_t = k / b$ for $t \in \{n_0, n_0 + T, n_0 + 2T, \ldots, n\}$ and 
$k_t = 0$ otherwise. 
This setting represents a middle ground between two extremes: where all $k$ queries are submitted in the initial round 
or final round. 
To provide a strong baseline from prior work, we treat the $b = 1$ case specially: it represents the static data 
setting, where the analyst waits for all $n$ data points to arrive before submitting all $k$ queries in a single batch 
at the final round $t = n$.

We consider the following generalization bounds:
{\setlength{\leftmargini}{1.25em}
\begin{itemize}
  \item \textbf{Ours-N}. \zcref{thm:gaus-mech-dist-acc} with non-uniform privacy.
  \item \textbf{Ours-U}. \zcref{thm:gaus-mech-dist-acc} with uniform privacy
  (expression $\vec{\delta}(t)$ is replaced by $\max_{\tau \in \irange{n}} \vec{\delta}(\tau)$ in our bounds).
  \item \textbf{JLNRSS}. \zcref{prop:gauss-mech-dist-acc-jlnrss} in \zcref{app:jlnrss-bound}. 
  Composes the static bound of \citet[Theorem~13]{jung2020new} over query batches using a 
  fresh static dataset for each query batch, thereby yielding a worst-case guarantee over all queries. 
  The static bound depends on parameters that are optimized under the constraints $\beta = \delta$ and 
  $c = d$. 
  \item \textbf{JLNRSS+}. A tighter variant of \textbf{JLNRSS} that differs in two aspects: 
  (1) the parameters are optimized without imposing the simplifying constraints mentioned above, 
  and (2) the conversion from zCDP to $(\epsilon, \delta)$-DP is based on the tighter result that we 
  use (\zcref{cor:zcdp-approx-dp-conversion}).
  \item \textbf{Split}. An analogue of the sampling splitting baseline from prior 
  work~\citep{bassily2016algorithmic,jung2020new} adapted for growing data. 
  It splits incoming data points into samples of size $\floor{n / k}$ and answers each query using a fresh sample. 
  Unlike the other methods, this method may respond with a delay if a query arrives before a fresh sample is ready. 
  Since there is no data reuse, a generalization bound follows directly from Hoeffding's bound and the union bound 
  (see \zcref{app:sample-split}).
\end{itemize}}

\zcref[S]{fig:K-vs-E-alpha-0.1-fixed-growth-ratio} plots the number of adaptive queries $k$ that can be answered as a 
function of the final dataset size $n$ while guaranteeing a confidence interval around estimates of 
$\alpha' = 0.1$ with uniform coverage probability $1 - \beta' = 0.95$, following the empirical settings of 
\citet{jung2020new}. 
When varying $n$, we set the initial dataset size $n_0 = n/3$ to maintain a constant growth ratio. 
This choice models a practical scenario where a substantial dataset is available before adaptive analysis begins, 
which is necessary to obtain non-vacuous guarantees in a fully adaptive setting.
To present the tightest possible guarantees for each method, we follow the approach of \citet{jung2020new} and 
numerically optimize over the free parameters of the bounds (e.g., $\sigma$, $\beta$, $\epsilon$ for 
\textbf{Ours-N} and \textbf{Ours-U}) for each point plotted. 
Our primary goal is to compare generalization guarantees under data reuse, so the resulting privacy parameters may 
vary slightly between points. 
We note, however, that the optimal values were found to be stable in practice, varying only beyond the first 
significant figure (e.g., $\sigma \approx 0.008$, $\beta \approx 10^{-5}$, $\epsilon \approx 0.04$).

We observe quadratic growth in $n$ for \textbf{Ours-U} and \textbf{Ours-N}, linear 
growth in $n$ for \textbf{Split} and slower quadratic growth in $n$ for \textbf{JLNRSS} and \textbf{JNLRSS+}.
In this regime, \textbf{Ours-N} generally outperforms the other bounds, which is not surprising as \textbf{JLNRSS} in 
particular is not optimized for growing data $b > 1$.
We provide additional results in \zcref{app:experiments} examining the error for a fixed number of queries; 
the effect of $b$; and a setting where $n_0$ is fixed and the growth ratio varies.

\section{Related Work} \label{sec:related}

Various methods have been proposed in the statistics community for adaptive analysis of static data. 
For instance, $\alpha$-investing and related methods can be used to control false discovery rates for sequential 
hypothesis testing~\citep{foster2008alpha,javanmard2018online}. 
Another line of work aims to ensure statistical validity when model selection and significance testing are performed 
on the same dataset~\citep{benjamini2010simultaneous,taylor2015statistical,fithian2017optimal}.  
However, these methods are specialized and place 
restrictions on the analyst~\citep{dwork2015preserving}. 

Our paper builds on a body of work exploiting a connection between stable algorithms and generalization for adaptive 
data analysis~\citep{dwork2015preserving,dwork2015generalization,bassily2016algorithmic,
feldman2018calibrating,shenfeld2019necessary,fish2020sampling,rogers2020guaranteed,jung2020new,dinur2023differential,
blanc2023subsampling,shenfeld2023generalization}. 
\citet{dwork2015preserving} were first to establish a \emph{transfer theorem} showing that differentially private 
algorithms are sufficient to guarantee high-probability bounds on the worst-case error of an adaptive analysis.
In subsequent work~\citep{bassily2016algorithmic,jung2020new}, simpler proofs of the transfer theorem were given 
that achieve sharper bounds, while covering a broader range of queries.
Recent work has obtained bounds that improve on the worst case, by conditioning on data\slash 
queries~\citep{feldman2018calibrating,rogers2020guaranteed,shenfeld2023generalization} or constraining the 
analyst~\citep{zrnic2019natural}. 
In a similar spirit, concurrent work also constrains the analyst's queries, which in turn allows them to provide 
guarantees for certain classes of correlated, non-i.i.d. data \citep{rapoport2025tight}.
Lower bounds have also been studied exploiting connections to cryptography~\citep{hardt2014preventing,
steinke2015interactive,nissim2023adaptive}.
However, all of this prior work is limited to static data.

While there is no prior work on adaptive analysis of dynamic data, this setting has been studied in 
differential privacy. 
One line of work is known as differential privacy \emph{under continual observation}~\citep{dwork2010differential,
chan2011private}, where the goal is to repeatedly estimate a fixed function of a dataset whenever new data 
arrives.
An elementary task in this setting is estimating the number of ones in a binary stream, which can be solved using 
the binary mechanism~\citep{dwork2010differential,chan2011private}. 
More recently, alternative mechanisms have been proposed that achieve tighter error 
bounds~\citep{henzinger2023almost,fichtenberger2023constant} while also maintaining computational 
efficiency~\citep{andersson2023smooth}. 
These counting mechanisms have been used as a primitive to tackle other tasks including frequency 
estimation~\citep{cardoso2022differentially,upadhyay2019sublinear}, learning~\citep{denisov2022improved,
henzinger2023almost} and graph spectrum analysis~\citep{upadhyay2021differentially}.
A second line of work studies differential privacy for more general kinds of adaptive queries on dynamic 
data~\citep{cummings2018differential,qiu2022differential}. 
\citet{cummings2018differential} design mechanisms for growing data that call black-box mechanisms for static data 
on a schedule. 
\citet{qiu2022differential} go beyond growing data, designing mechanisms that estimate adaptive linear queries 
on datasets where items may be inserted or deleted over time. 
Our paper provides generalization guarantees for many of these mechanisms.

Several works have obtained generalization bounds for adaptive data analysis that do not rely on differential privacy. 
While differential privacy yields accuracy bounds that hold with high probability, one can also study weaker bounds 
that hold on average via connections to information theory~\citep{russo2016controlling,bassily2016algorithmic,
steinke2020reasoning}. 
Concepts from computational learning theory, such as Rademacher complexity, have also been used to obtain 
data-dependent generalization bounds for adaptive testing~\citep{destefani2019rademacher}. 
However, estimating data-dependent bounds on Rademacher complexity may be computationally challenging. 

Connections have been made between adaptive data analysis and seemingly disparate areas. 
\citet{steinke2023privacy} develop a method for privacy auditing that runs the algorithm under audit once on a 
single dataset rather than many times on adjacent datasets. 
The analysis of their method relies on generalization bounds for adaptive data analysis tailored for uniformly 
distributed binary data. 
\citet{liu2024program} use static and dynamic analysis to estimate the adaptivity of a program to assist in bounding 
its generalization error.

\section{Conclusion} \zlabel{sec:conclusion}

This paper extends the current understanding of generalization in adaptive data analysis to dynamic scenarios where 
data arrives incrementally over time, a setting increasingly relevant in many data-driven fields. 
Our approach builds on and extends the tightest known worst-case generalization guarantees for static 
data~\citep{jung2020new} by incorporating time-varying accuracy bounds and addressing the additional complexity 
introduced by data growth. 
Compared with bounds for static data, our bounds incorporate an additional factor proportional to the 
data growth, associated with the stability of the analysis as measured by differential privacy.
We instantiate our bounds for three query classes and demonstrate an empirical improvement over baselines for 
adaptive statistical queries answered with the clipped Gaussian mechanism. 

There are various opportunities to extend our work. 
While it is conventional to assume i.i.d.\ data when studying generalization, as we have done here, it would be 
interesting to consider non-i.i.d.\ growing data where the data distribution evolves over time. 
This may require bounds that depend on the rate of evolution, akin to bounds that depend on the correlation for 
non-i.i.d. data in the static setting~\citep{kontorovich2022adaptive}.
Another practical direction is to tighten our (mostly) worst-case bounds by conditioning on the actual queries and 
data realized in the analysis.
This could be informed by similar work for the static setting~\citep{feldman2018calibrating,rogers2020guaranteed,
shenfeld2023generalization}.
Finally, we believe there are opportunities to design new differentially private mechanisms for different kinds of 
adaptive queries on dynamic data, as there has been limited work in this area to 
date~\citep{cummings2018differential,qiu2022differential}.

\begin{ack}
    We acknowledge support from the Australian Research Council Discovery Project DP220102269.
\end{ack}

\printbibliography

\appendix

\section{Results for Minimization Queries} \zlabel{app:gen-stability-min-queries}

In this appendix, we provide a generalization guarantee for the class of low-sensitivity minimization queries. 
Queries in this class are parameterized by a loss function $L: \scrX^* \times \Theta \to [0, 1]$, where the first 
argument is a data snapshot and the second argument is a set of parameters from a parameter space $\Theta$.  
We require that $L$ is $\vec{\Delta}$-sensitive in its first argument, meaning for all $t \in \irange{n}$ and 
all $\theta \in \Theta$ we have 
$\abs{L(\vec{x}_{\irange{t}}, \theta) - L(\vec{x}_{\irange{t}}', \theta)} \leq \Delta_t$ for any pair of neighboring 
snapshots $\vec{x}_{\irange{t}}, \vec{x}_{\irange{t}}' \in \scrX^t$.
When evaluated on a snapshot, the result of the query is 
$q(\vec{x}_{\irange{t}}) \in \argmin_{\theta \in \Theta} L(\vec{x}_{\irange{t}}, \theta)$.
When evaluated on a data distribution $\scrD$, the result is 
$q(\scrD) \in \argmin_{\theta \in \Theta} \E_{\vec{X} \sim \scrD} [L(\vec{X}, \theta)]$. 

The definitions of distributional and snapshot accuracy that appear in \zcref{sec:gen-stability} must be adapted for 
minimization queries. 
Minimization queries require a different treatment because the output of the query is not scalar-valued in general, 
but rather a set of parameters $\theta \in \Theta$. 
We use the loss function $L$ associated with the query to measure the distributional and snapshot accuracy as defined 
below.

\begin{definition} \zlabel{def:distributional-acc-min-queries}
  Let $\alpha_t \geq 0$ for all rounds $t \in \irange{n_0, n}$ and $\beta \geq 0$. 
  A mechanism $\mech$ is $(\{\alpha_t\}, \beta)$-\emph{distributionally accurate} if for all analysts $\analyst$ 
  and data distributions $\scrP$ 
  \begin{equation}
    \Pr_{\vec{X} \sim \scrP^n, \Pi \sim \sinteract( \vec{X}; \analyst, \mech)} \Biggl(
      \bigcup_{t = n_0}^{n} \bigcup_{(L, \theta) \in \Pi_t} \left\{\abs*{
        \E_{\tilde{\vec{X}} \sim \scrP^t} [L(\tilde{\vec{X}}, \theta)]
        - \min_{\tilde{\theta} \in \Theta} \E_{\tilde{\vec{X}} \sim \scrP^t} [L(\tilde{\vec{X}}, \tilde{\theta})]
      } \geq \alpha_t \right\} 
    \Biggr) \leq \beta.
  \end{equation}
\end{definition}

\begin{definition} \zlabel{def:snapshot-acc-min-queries}
  Let $\alpha_t \geq 0$ for all rounds $t \in \irange{n_0, n}$ and $\beta \geq 0$. 
  A mechanism $\mech$ is $(\{\alpha_t\}, \beta)$-\emph{snapshot accurate} if for all analysts $\analyst$ and data 
  distributions $\scrP$
  \begin{equation}
    \Pr_{\vec{X} \sim \scrP^n, \Pi \sim \sinteract( \vec{X}; \analyst, \mech)} \Biggl(
      \bigcup_{t = n_0}^{n} \bigcup_{(L, \theta) \in \Pi_t} \left\{\abs*{
        L(\vec{X}_{\irange{t}}, \theta) 
        - \min_{\tilde{\theta} \in \Theta} L(\vec{X}_{\irange{t}}, \tilde{\theta})
      } \geq \alpha_t \right\} 
    \Biggr) \leq \beta.
  \end{equation}
\end{definition}

Following prior work~\citep{bassily2016algorithmic,jung2020new}, we obtain a generalization guarantee for $\vec{\Delta}$-sensitive minimization queries by applying \zcref{thm:transfer-low-sens-q} to a related set of 
$2 \vec{\Delta}$-sensitive scalar-valued queries. 

\begin{theorem} \zlabel{thm:transfer-min-q}
  Suppose $\mech$ is an $(\{\alpha_t\}, \beta)$-snapshot accurate mechanism and $\sinteract(\blank; \blank, \mech)$ 
  is an $(\epsilon, \vec{\delta})$-DP interaction for $\vec{\Delta}$-sensitive minimization queries. 
  Then for any constants $c, d > 0$, $\mech$ is $(\{\alpha'_t\}, \beta')$-distributionally accurate for 
  $\alpha_t' = \alpha_t + 2 \euler^{\epsilon} \max_{\tau_1 \in \irange{n_0, n}} \tau_1 \Delta_{\tau_1} - 2 
      \min_{\tau_2 \in \irange{n_0, n}} \tau_2 \Delta_{\tau_2} 
      + 8 c \sum_{t = 1}^{n} \vec{\delta}(t) \max_{\tau_3 \in \irange{n_0, n}} \Delta_{\tau_3} + d$
  and $\beta' = \beta / d + 1 / c$.
\end{theorem}
\begin{proof}
  We begin by defining a mapping $f: \scrQ_\text{min} \times \Theta \to \scrQ_\text{ls} \times [0, 1]$ that takes a 
  $\vec{\Delta}$-sensitive minimization query-estimate pair $(L, W) \in \scrQ_\text{min} \times \Theta$ and returns 
  a $2 \vec{\Delta}$-sensitive scalar-valued query-estimate pair $(q, r) \in \scrQ_\text{ls} \times [0,1]$:
  \begin{align}
    f(L, W) = (q, r) \quad \text{with} \quad
    q(\vec{x}) \coloneqq  L(\vec{x}, W) - \min_{W' \in \Theta} L(\vec{x}, W) \quad \text{and} \quad
    r \coloneqq 0.
  \end{align}
  Since $f$ does not depend on the dataset, we can apply it to all pairs in the minimization query transcript 
  $\Pi$ to yield a transformed transcript $\Pi'$ that also satisfies $(\epsilon, \vec{\delta})$-DP by the 
  post-processing guarantee (\zcref{thm:approx-dp-post-processing}).
  It is straightforward to see that the transformed transcript $\Pi'$ is $(\{\alpha_t\}, \beta)$-snapshot accurate iff 
  the original transcript $\Pi$ is.
  
  Next, observe that the probability of interest can be upper-bounded using Jensen's inequality to swap the order 
  of the minimization and expectation operations:
  \begin{align}
    & \Pr_{\vec{X} \sim \scrP^n, \Pi \sim \sinteract(\vec{X})} \Biggl(
    \bigcup_{t = n_0}^{n} \bigcup_{(L, W) \in \Pi_t} \left\{ \abs*{
      \E_{\vec{X}' \sim \scrP^t} \left[ L(\vec{X}', W) \right] 
      - \min_{W' \in \Theta} \E_{\vec{X}' \sim \scrP^t} \left[ L(\vec{X}', W') \right]} > \alpha_t'
    \right\}
    \Biggr) \\
    & \quad \leq \Pr_{\vec{X} \sim \scrP^n, \Pi \sim \sinteract(\vec{X})} \Biggl(
    \bigcup_{t = n_0}^{n} \bigcup_{(L, W) \in \Pi_t} \left\{ \abs*{
      \E_{\vec{X}' \sim \scrP^t} \left[ L(\vec{X}', W)
      - \min_{W' \in \Theta} L(\vec{X}', W') \right]} > \alpha_t'
    \right\}
    \Biggr) \\
    & \quad = \Pr_{\vec{X} \sim \scrP^n, \Pi \sim \sinteract(\vec{X})} \Biggl(
    \bigcup_{t = n_0}^{n} \bigcup_{(q, r) \in \Pi_t'} \left\{ \abs*{
      r - \E_{\vec{X}' \sim \scrP^t} \left[ q(\vec{X}') \right]} > \alpha_t'
    \right\}
    \Biggr).
  \end{align}
  In the last line above, we have rewritten the event in terms of the transformed transcript $\Pi'$.
  Applying \zcref{thm:transfer-low-sens-q} upper bounds this probability by $\beta'$, completing the proof.
\end{proof}

\section{Proofs for \texorpdfstring{\zcref{sec:theory}}{Section~\ref{sec:theory}}} \zlabel{app:proofs-theory}

\begin{lemma}[Resampling Lemma \citealp{jung2020new}] \zlabel{lem:resampling}
  Let $E \subseteq \scrX^n \times \scrT$ be any event. Then
  \begin{equation}
    \Pr_{\vec{X} \sim \scrP^n, \Pi \sim \sinteract(\vec{X})} \mathopen{} \left((\vec{X}, \Pi) \in E\right)
    = \Pr_{\vec{X} \sim \scrP^n, \Pi \sim \sinteract(\vec{X}), \vec{X}' \sim \scrQ_{\Pi}} \mathopen{} \left((\vec{X}', 
    \Pi) \in 
    E\right).
  \end{equation}
\end{lemma}
\begin{proof}
  The result follows by writing the probabilities as expectations, invoking the definition of $\scrQ_\Pi$, and 
  using the fact that $\vec{x}$ and $\vec{x}'$ can be swapped without changing the expectation:
  \begin{align}
    \Pr_{\vec{X} \sim \scrP^n, \Pi \sim \sinteract(\vec{X}), \vec{X}' \sim \scrQ_{\Pi}} \mathopen{} \left(
    (\vec{X}', \Pi) \in E
    \right) 
    &= \sum_{\vec{x}, \pi, \vec{x}'} \Pr(\vec{X} = \vec{x}, \Pi = \pi) \Pr_{\vec{X}' \sim \scrQ_\pi}(\vec{X}' = 
    \vec{x}')
    \ind{(\vec{x}', \pi) \in E} \\ 
    &= \sum_{\pi, \vec{x}'} \Pr(\Pi = \pi) \Pr(\vec{X} = \vec{x}' \mid \Pi = \pi) \ind{(\vec{x}', \pi) \in E} \\ 
    &= \sum_{\vec{x}, \pi} \Pr(\vec{X} = \vec{x}, \Pi = \pi) \ind{(\vec{x}, \pi) \in E} \\
    &= \Pr_{\vec{X} \sim \scrP^n, \Pi \sim \sinteract(\vec{X})} \mathopen{} \left((\vec{X}, \Pi) \in E\right).
  \end{align}
\end{proof}

\lemcondacc*
\begin{proof}
  Given a transcript $\pi \in \scrT$, let a round-query-estimate tuple that achieves the largest 
  $\alpha_t$-adjusted posterior error be denoted
  \begin{equation}
    t^\star, q^\star, r^\star \in \argmax_{t \in \irange{n_0, n}, (q,r) \in \pi_t} \abs{r - q(\scrQ_\pi^t)} - \alpha_t,
  \end{equation}
  where we have omitted the dependence on $\pi$. 
  We use this definition to write the probability of interest in terms of a single event, which we then express as a 
  union of two independent (since $\alpha_{t^\star} + c> 0$) events corresponding to the branches of the absolute 
  value function:
  \begin{align}
    & \Pr_{\vec{X} \sim \scrP^n, \Pi \sim \sinteract(\vec{X})} \Biggl(
    \bigcup_{t = n_0}^{n} \bigcup_{(q, r) \in \Pi_t} \left\{ \abs{r - q(\scrQ_{\Pi}^t)} > \alpha_t + c \right\} 
    \Biggr) \\
    & \quad = \Pr_{\vec{X} \sim \scrP^n, \Pi \sim \sinteract(\vec{X})} \mathopen{} \left(
    \abs{r^\star - q^\star(\scrQ_{\Pi}^{t^\star})} - \alpha_{t^\star} > c 
    \right) \\
    & \quad = \Pr_{\vec{X} \sim \scrP^n, \Pi \sim \sinteract(\vec{X})} \mathopen{} \left(
    r^\star - q^\star(\scrQ_{\Pi}^{t^\star}) - \alpha_{t^\star} > c 
    \right) + \Pr_{\vec{X} \sim \scrP^n, \Pi \sim \sinteract(\vec{X})} \mathopen{} \left(
    q^\star(\scrQ_{\Pi}^{t^\star}) - r^\star - \alpha_{t^\star} > c 
    \right). \label{eqn:lem6-post-err-two-terms}
  \end{align}
  
  Observe that the first term in \eqref{eqn:lem6-post-err-two-terms} can be bounded as follows:
  \begin{align}
    &\Pr_{\vec{X} \sim \scrP^n, \Pi \sim \sinteract(\vec{X})} \mathopen{} \left(
    r^\star - q^\star(\scrQ_{\Pi}^{t^\star}) - \alpha_{t^\star} > c
    \right) \\
    & \quad = \Pr_{\vec{X} \sim \scrP^n, \Pi \sim \sinteract(\vec{X})} \mathopen{} \left(
    \E_{\vec{X}' \sim \scrQ_{\Pi}} \left[r^\star - q^\star(\vec{X}_{\irange{t^\star}}') - \alpha_{t^\star} \right] > c
    \right) \label{eqn:lem6-post-def} \\
    & \quad \leq \Pr_{\vec{X} \sim \scrP^n, \Pi \sim \sinteract(\vec{X})} \mathopen{} \left(
    \E_{\vec{X}' \sim \scrQ_{\Pi}} \left[ \max \{r^\star - q^\star(\vec{X}_{\irange{t^\star}}') - \alpha_{t^\star}, 0\} 
    \right] > c
    \right) \label{eqn:lem6-jensen} \\
    & \quad \leq \frac{1}{c} \E_{\vec{X} \sim \scrP^n, \Pi \sim \sinteract(\vec{X})} \mathopen{} \left[ 
    \E_{\vec{X}' \sim \scrQ_{\Pi}} \left[ \max \{r^\star - q^\star(\vec{X}_{\irange{t^\star}}') - \alpha_{t^\star}, 0\} 
    \right]
    \right] \label{eqn:lem6-markov} \\ %
    & \quad \leq \frac{1}{c} \E_{\vec{X} \sim \scrP^n, \Pi \sim \sinteract(\vec{X})} \mathopen{} \left[ 
    \Pr_{\vec{X}' \sim \scrQ_{\Pi}} \mathopen{} \left( r^\star - q^\star(\vec{X}_{\irange{t^\star}}') - 
    \alpha_{t^\star} > 0 \right)
    \right] \label{eqn:lem6-bounded} \\ %
    & \quad = \frac{1}{c} \Pr_{\vec{X} \sim \scrP^n, \Pi \sim \sinteract(\vec{X}), \vec{X}' \sim \scrQ_{\Pi}} 
    \mathopen{} \left( 
    r^\star - q^\star(\vec{X}_{\irange{t^\star}}') - \alpha_{t^\star} > 0
    \right) \\ %
    & \quad = \frac{1}{c} \Pr_{\vec{X} \sim \scrP^n, \Pi \sim \sinteract(\vec{X})} \mathopen{} \left(
    r^\star - q^\star(\vec{X}_{\irange{t^\star}}) - \alpha_{t^\star} > 0
    \right) \label{eqn:lem6-resample} %
  \end{align}
  where line~\eqref{eqn:lem6-post-def} follows from the definition of $q_{t,j}(\scrQ_{\Pi}^t)$;
  line~\eqref{eqn:lem6-markov} follows from Markov's inequality; 
  line~\eqref{eqn:lem6-bounded} follows from the fact that $r - q(\vec{X}_{\irange{t}}') - \alpha_t \leq 1$ 
  for a $[0,1]$-bounded query and mechanism; 
  and line~\eqref{eqn:lem6-resample} follows from \zcref{lem:resampling}.
  By symmetry, a similar bound holds for the second term in \eqref{eqn:lem6-post-err-two-terms}. 
  
  Substituting these bounds in \eqref{eqn:lem6-post-err-two-terms} gives
  \begin{align}
    & \Pr_{\vec{X} \sim \scrP^n, \Pi \sim \sinteract(\vec{X})} \Biggl(
    \bigcup_{t = n_0}^{n} \bigcup_{(q, r) \in \Pi_t} \left\{ \abs{r - q(\scrQ_{\Pi}^t)} > \alpha_t + c \right\} 
    \Biggr) \\
    & \quad \leq \frac{1}{c} \Pr_{\vec{X} \sim \scrP^n, \Pi \sim \sinteract(\vec{X})} \mathopen{} \left(
    \abs{r^\star - q^\star(\vec{X}_{\irange{t^\star}})} - \alpha_{t^\star} > 0
    \right) \label{eqn:lem6-abs-value} \\
    & \quad \leq \frac{1}{c} \Pr_{\vec{X} \sim \scrP^n, \Pi \sim \sinteract(\vec{X})} \mathopen{} \Biggl(
    \bigcup_{t = n_0}^{n} \bigcup_{(q,r) \in \Pi_t} \{\abs{r - q(\vec{X}_{\irange{t}})} > \alpha_t\} 
    \Biggr) \label{eqn:lem6-star-suboptimal} \\
    & \quad \leq \frac{\beta}{c}, \label{eqn:lem6-snapshot-acc}
  \end{align}
  where line~\eqref{eqn:lem6-abs-value} follows from the independence of the events in the two terms; 
  line~\eqref{eqn:lem6-star-suboptimal} follows since the starred round-query-estimate tuple may not achieve the 
  largest $\alpha_t$-adjusted snapshot error; 
  and line~\eqref{eqn:lem6-snapshot-acc} follows from the definition of $(\{\alpha_t\}, \beta)$-snapshot accuracy.
\end{proof}

\thmtransfergen*
\begin{proof}
  Given a transcript $\pi \in \scrT$, let a round-query-estimate tuple that achieves the largest 
  $\alpha_t$-adjusted distributional error be denoted
  \begin{equation}
    t^\star, q^\star, r^\star \in \argmax_{t \in \irange{n_0, n}, (q,r) \in \pi_t} \abs{r - q(\scrP^t)} - \alpha_t,
  \end{equation}
  where we have omitted the dependence on $\pi$. 
  Using this definition, we express the probability of interest in terms of a single event, and then obtain an 
  upper bound using the triangle inequality and the union bound:
  \begin{align}
    & \Pr_{\vec{X} \sim \scrP^n, \Pi \sim \sinteract(\vec{X})} \Biggl(
    \bigcup_{t = n_0}^{n} \bigcup_{(q,r) \in \Pi_t} \{ \abs{r - q(\scrP^t)} > \alpha_t + c + \epsilon \} 
    \Biggr) \\
    & \quad = \Pr_{\vec{X} \sim \scrP^n, \Pi \sim \sinteract(\vec{X})} \mathopen{} \left(
    \abs{r^\star - q^\star(\scrP^{t^\star})} - \alpha_{t^\star} > c + \epsilon
    \right) \\
    & \quad \leq \Pr_{\vec{X} \sim \scrP^n, \Pi \sim \sinteract(\vec{X})} \mathopen{} \left(
    \abs{r^\star - q^\star(\scrQ_{\Pi}^{t^\star})} - \alpha_{t^\star} 
    + \abs{q^\star(\scrQ_{\Pi}^{t^\star}) - q^\star(\scrP^{t^\star})} > c + \epsilon
    \right) \label{eqn:transfer-triangle} \\
    & \quad \leq \Pr_{\vec{X} \sim \scrP^n, \Pi \sim \sinteract(\vec{X})} \mathopen{} \left(
    \abs{r^\star - q^\star(\scrQ_{\Pi}^{t^\star})} - \alpha_{t^\star} > c
    \right) \\
    & \quad \qquad {} + \Pr_{\vec{X} \sim \scrP^n, \Pi \sim \sinteract(\vec{X})} \mathopen{} \left(
    \abs{q^\star(\scrQ_{\Pi}^{t^\star}) - q^\star(\scrP^{t^\star})} > \epsilon
    \right). \label{eqn:transfer-union}
  \end{align}
  We can upper bound the two probabilities in \eqref{eqn:transfer-union} by maximizing the LHS of the inequalities 
  with respect to $t^\star, q^\star, r^\star \in \bigcup_{t = n_0}^{n} \{(t, q, r) : (q, r) \in \Pi_t\}$. 
  Then \zcref{lem:cond-acc} upper bounds the first probability by $\frac{\beta}{c}$ and 
  \zcref{def:posterior-stability} bounds the second probability by $\delta$, giving the required result.
\end{proof}

\lemdppsstat*
\begin{proof}
  Given a transcript $\pi \in \scrT$, let the round-query-estimate tuple that achieves the largest absolute 
  difference be denoted
  \begin{equation}
    t^\star, q^\star, r^\star
    \in \argmax_{t \in \irange{n_0, n}, (q, r) \in \pi_t} \abs{q(\scrQ_\pi^t) - q(\scrP^t)},
  \end{equation}
  where we have omitted the dependence on $\pi$.
  
  Define for any $\alpha > 0$, $x \in \scrX$, $t \in \irange{n}$, $z \in [0, 1/n_0]$:
  \begin{align}
    \vec{\Pi}(\alpha) &= \left\{ 
    \pi \in \scrT : q^\star(\scrQ_{\pi}^{t^\star}) - q^\star(\scrP^{t^\star}) > \alpha 
    \right\}, \\
    \scrX^+(\pi) &= \left\{
    x \in \scrX : \Pr_{\vec{X}' \sim \scrQ_{\pi}, t \sim \unif_{\irange{n_0, t^\star}}}(X_{\tau}' = x)
    \geq \Pr_{X \sim \scrP}(X = x)
    \right\}, \\
    B^{+}(\alpha) &= \bigcup_{\pi \in \vec{\Pi}(\alpha)} (\scrX^+(\pi) \times \{\pi\}), \\
    \vec{\Pi}^+(\alpha, x) &= \{\pi \in \scrT: (x, \pi) \in B^{+}(\alpha)\}, \\
    \vec{\Pi}^+(\alpha, x, z, t) &= \left\{
    \pi \in \vec{\Pi}^+(\alpha, x) : t \leq t^\star,  1 > z t^\star \right\}. \\
  \end{align}
  
  Fix any $\alpha > 0$ and let $\tilde{\delta} = \sum_{t = 1}^{n} \vec{\delta}(t)$.
  Suppose 
  $\Pr_{\vec{X} \sim \scrP^n, \Pi \sim \sinteract(\vec{X})} \left(
  \abs{q^\star(\scrQ_{\Pi}^{t^\star}) - q^\star(\scrP^{t^\star})} > \alpha 
  \right) > \frac{1}{c}$, which implies either 
  $\Pr_{\vec{X} \sim \scrP^n, \Pi \sim \sinteract(\vec{X})} \left(
  q^\star(\scrQ_{\Pi}^{t^\star}) - q^\star(\scrP^{t^\star}) > \alpha 
  \right) > \frac{1}{2c}$ or 
  $\Pr_{\vec{X} \sim \scrP^n, \Pi \sim \sinteract(\vec{X})} \left(
  q^\star(\scrP^{t^\star}) - q^\star(\scrQ_{\Pi}^{t^\star}) > \alpha 
  \right) > \frac{1}{2c}$.
  Without loss of generality assume 
  \begin{equation}
    \Pr_{\vec{X} \sim \scrP^n, \Pi \sim \sinteract(\vec{X})} \left(
    q^\star(\scrQ_{\Pi}^{t^\star}) - q^\star(\scrP) > \alpha
    \right) 
    = \Pr(\Pi \in \vec{\Pi}(\alpha))
    > \frac{1}{2c}. \label{eqn:lem7-delta-inequal}
  \end{equation}
  
  From the definition of $\vec{\Pi}(\alpha)$, we have
  \begin{align}
    \alpha \Pr(\Pi \in \vec{\Pi}(\alpha)) 
    & < \sum_{\pi \in \vec{\Pi}(\alpha)} 
    \{ q^\star(\scrQ_{\pi}^{t^\star}) - q^\star(\scrP^{t^\star}) \}
    \\ %
    & = \sum_{\pi \in \vec{\Pi}(\alpha)} \Big\{
    \E_{\vec{X}' \sim \scrQ_{\pi}} [q^\star(\vec{X}_{\irange{t^\star}}')] 
    - \E_{\vec{X} \sim \scrP^n}[q^\star(\vec{X}_{\irange{t^\star}})]
    \Big\} \\ %
    & = \sum_{\pi \in \vec{\Pi}(\alpha)} \Big\{
    \E_{\vec{X}' \sim \scrQ_{\pi}, t \sim \unif_{\irange{t^\star}}} [q^\star(X_{t}')] 
    - \E_{X \sim \scrP}[q^\star(X)]
    \Big\}, \label{eqn:lem7-linearq} %
  \end{align}
  where we have used linearity of statistical queries in the last line.
  Expanding out the difference of expectations, we have
  \begin{align}
    &\E_{\vec{X}' \sim \scrQ_{\pi}, t \sim \unif_{\irange{t^\star}}} [q^\star(X_{t}')] 
    - \E_{X \sim \scrP}[q^\star(X)] \\
    &\quad = \sum_{x \in \scrX} q^\star(x) \left\{ 
    \frac{1}{t^\star} \sum_{t = 1}^{t^\star} \Pr(X_{t}' = x \mid \Pi = \pi) - \Pr(X = x)
    \right\} \\
    &\quad \leq \sum_{x \in \scrX^+(\pi)} \left\{ 
    \frac{1}{t^\star} \sum_{t = 1}^{t^\star} \Pr(X_{t}' = x \mid \Pi = \pi) - \Pr(X = x)
    \right\} \label{eqn:lem7-bound-query}\\ %
    &\quad = \sum_{x \in \scrX^+(\pi)} \frac{\Pr(X = x)}{\Pr(\Pi = \pi)} \left\{
    \frac{1}{t^\star} \sum_{t = 1}^{t^\star} \Pr(\Pi = \pi \mid X_{t}' = x) - \Pr(\Pi = \pi)
    \right\} \label{eqn:lem7-bayes} \\ %
  \end{align}
  where \eqref{eqn:lem7-bound-query} follows by dropping negative terms from the sum and using the boundedness of 
  statistical queries,
  and \eqref{eqn:lem7-bayes} follows from Bayes' theorem.
  
  Putting \eqref{eqn:lem7-bayes} in \eqref{eqn:lem7-linearq} and swapping the order of the sums gives
  \begin{align}
    &\alpha \Pr(\Pi \in \vec{\Pi}(\alpha)) \\
    &< \sum_{\pi \in \vec{\Pi}(\alpha)} \sum_{x \in \scrX^+(\pi)} \Pr(X = x) \left\{
    \frac{1}{t^\star} \sum_{t = 1}^{t^\star} \Pr(\Pi = \pi \mid X_{t}' = x) - \Pr(\Pi = \pi)
    \right\} \\
    &= \sum_{t = 1}^{n} \sum_{x \in \scrX} \Pr(X = x) \sum_{\pi \in \vec{\Pi}^+(\alpha, x)} 
    \frac{1}{t^\star} \ind{t \leq t^\star} \left\{ \Pr(\Pi = \pi \mid X_{t}' = x) - \Pr(\Pi = \pi) 
    \right\} \\ %
    &= \sum_{t = 1}^{n} \sum_{x \in \scrX} \Pr(X = x) \sum_{\pi \in \vec{\Pi}^+(\alpha, x)}  
    \int_{0}^{\frac{1}{n_0}} \ind{\frac{1}{t^\star} > z} \, dz \, \ind{t \leq t^\star} 
    \left\{ \Pr(\Pi = \pi \mid X_{t}' = x) - \Pr(\Pi = \pi) \right\} \\ %
    &= \sum_{t = 1}^{n} \sum_{x \in \scrX} \Pr(X = x) \int_{0}^{\frac{1}{n_0}} \left\{
    \Pr(\Pi \in \vec{\Pi}^+(\alpha, x, z, t) \mid X_{t}' = x) - \Pr(\Pi \in \vec{\Pi}^+(\alpha, x, z, t)) 
    \right\} \, dz \\ %
    &\leq \sum_{t = 1}^{n} \sum_{x \in \scrX} \Pr(X = x) \int_{0}^{\frac{1}{n_0}} \left\{
    (\euler^{\epsilon} - 1) \Pr(\Pi \in \vec{\Pi}^+(\alpha, x, z, t)) + \vec{\delta}(t)
    \right\} \, dz \label{eqn:lem7-dp} \\ %
    &= \sum_{x \in \scrX} \Pr(X = x) \left\{
    (\euler^{\epsilon} - 1) \Pr(\Pi \in \vec{\Pi}^+(\alpha, x)) + \frac{\tilde{\delta}}{n_0} 
    \right\} \\ %
    &= (\euler^{\epsilon} - 1) \Pr(\Pi \in \vec{\Pi}(\alpha)) + \frac{\tilde{\delta} }{n_0} \\ %
    &< \left(\euler^{\epsilon} - 1 + \frac{2 c \tilde{\delta}}{n_0}\right) \Pr(\Pi \in \vec{\Pi}(\alpha))
    \label{eqn:lem7-delta} %
  \end{align}
  where \eqref{eqn:lem7-dp} follows from \zcref{lem:dp}, and \eqref{eqn:lem7-delta} follows from 
  \eqref{eqn:lem7-delta-inequal}.
  This is a contradiction for $\alpha \geq \euler^{\epsilon} - 1 + 2 c \tilde{\delta} / n_0$.
\end{proof}

\lemdppslowsens*
\begin{proof}
  Given a transcript $\pi \in \scrT$, let a round-query-result tuple that achieves the largest absolute difference 
  be denoted
  \begin{equation}
    t^\star, q^\star, r^\star \in \argmax_{t \in \irange{n_0, n}, (q, r) \in \pi_t} \abs{q(\scrQ_\pi^t) - q(\scrP^t)},
  \end{equation}
  where the dependence on $\pi$ is omitted.
  
  Let $\concat$ denote the concatenation operator, such that for any pair of sequences $\vec{x}$ and $\vec{y}$ of 
  length $n$ and $m$ respectively, we have $\vec{x} \concat \vec{y} \coloneqq (x_1, \ldots, x_{n}, y_1, \ldots, y_{m})$. 
  Let $\Delta^\star = \max_{t \in \irange{n_0, n}} \Delta_t$ and 
  $\tilde{\delta} = \sum_{t = 1}^{n} \vec{\delta}(t)$.
  For any $\alpha \geq 0$, $\tau, t \in \irange{n}$, $\vec{x} \in \scrX^n$ and $z \in [0, 2 \Delta^\star]$, define
  \begin{align}
    v(q, \tau, \vec{x}_{\irange{t}}) &\coloneqq \E_{\vec{X}' \sim \scrP^{\max\{\tau - t, 0\}}} \left[
    q(\vec{x}_{\irange{t}} \concat \vec{X}')
    \right] \\
    \vec{\Pi}(\alpha) &\coloneqq \left\{ 
    \pi \in \scrT : q^\star(\scrQ_\pi^{t^\star}) - q^\star(\scrP^{t^\star}) > \alpha 
    \right\}, \\
    \vec{\Pi}(\alpha, z, \vec{x}_{\irange{t}}) &\coloneqq \left\{  
    \pi \in \vec{\Pi}(\alpha) : 
    \ind{t \leq t^\star} \left(
    v(q^\star, t^\star, \vec{x}_{\irange{t}}) - v(q^\star, t^\star, \vec{x}_{\irange{t - 1}}) + \Delta_{t^\star}
    \right) > z  
    \right\}.
  \end{align}
  Using the definition of differential privacy, observe that
  \begin{align}
    & \sum_{\pi \in \vec{\Pi}(\alpha)} \Pr(\Pi = \pi \mid \vec{X} = \vec{x})
    \ind{t \leq t^\star} \left(
    v(q^\star, t^\star, \vec{x}_{\irange{t}}) - v(q^\star, t^\star, \vec{x}_{\irange{t - 1}}) + \Delta_{t^\star}
    \right) \\ 
    &\quad = \sum_{\pi \in \vec{\Pi}(\alpha)} \Pr(\Pi = \pi \mid \vec{X} = \vec{x})
    \int_{0}^{2 \Delta^\star} \ind{\ind{t \leq t^\star} (
      v(q^\star, t^\star, \vec{x}_{\irange{t}}) - v(q^\star, t^\star, \vec{x}_{\irange{t - 1}}) + \Delta_{t^\star}
      ) > z} \, dz \\ 
    &\quad = \int_{0}^{2 \Delta^\star} \Pr(\Pi \in \vec{\Pi}(\alpha, z, \vec{x}_{\irange{t}}) \mid \vec{X} = \vec{x}) \, dz 
    \\ 
    &\quad \leq \int_{0}^{2 \Delta^\star} \left(
    e^{\epsilon} \Pr(\Pi \in \vec{\Pi}(\alpha, z, \vec{x}_{\irange{t}}) \mid \vec{X} = \replace(\vec{x}, t, x')) + 
    \vec{\delta}(t)
    \right) \, dz \\ 
    &\quad = e^{\epsilon} \sum_{\pi \in \vec{\Pi}(\alpha)} \Pr(\Pi = \pi \mid \vec{X} = \replace(\vec{x}, t, x'))
    \ind{t \leq t^\star} \left(
    v(q^\star, t^\star, \vec{x}_{\irange{t}}) - v(q^\star, t^\star, \vec{x}_{\irange{t - 1}}) + \Delta_{t^\star}
    \right) \\
    & \quad \qquad {} + 2 \Delta^\star \vec{\delta}(t) \label{eqn:lem-dp-ps-low-sens-dp}
  \end{align}
  where $\replace(\vec{x}, t, x')$ denotes the dataset obtained from $\vec{x}$ by replacing the $t$-th data point with 
  $x' \in \scrX$.
  Taking the expectation of inequality \eqref{eqn:lem-dp-ps-low-sens-dp} with respect to $\vec{X} \sim \scrP^{n}$ and 
  $X' \sim \scrP$, and summing over $t \in \irange{n}$, we have 
  \begin{align}
    &\sum_{t = 1}^{n} \E_{\vec{X} \sim \scrP^n} \Biggl[
    \sum_{\pi \in \vec{\Pi}(\alpha)} \Pr(\Pi = \pi \mid \vec{X})
    \ind{t \leq t^\star} \left(
    v(q^\star, t^\star, \vec{X}_{\irange{t}}) - v(q^\star, t^\star, \vec{X}_{\irange{t - 1}}) + \Delta_{t^\star}
    \right) 
    \Biggr] \\ 
    &\quad \leq \sum_{t = 1}^{n} \E_{\vec{X} \sim \scrP^n, X \sim \scrP} \Biggl[
    e^{\epsilon} \sum_{\pi \in \vec{\Pi}(\alpha)} \Pr(\Pi = \pi \mid \vec{X}) \\
    &\qquad {} \times \ind{t \leq t^\star} \left(
    v(q^\star, t^\star, \replace(\vec{x}, t, x')) - v(q^\star, t^\star, \vec{x}_{\irange{t - 1}}) + \Delta_{t^\star}
    \right) + 2 \Delta^\star \delta 
    \Biggr] \label{eqn:lem-dp-ps-low-sens-same-dist} \\
    &\quad \leq \sum_{t = 1}^{n} \E_{\vec{X} \sim \scrP^n} \Biggl[
    e^{\epsilon} \sum_{\pi \in \vec{\Pi}(\alpha)} \Pr(\Pi = \pi \mid \vec{X}) \ind{t \leq t^\star} 
    \Delta_{t^\star} + 2 \Delta^\star \vec{\delta}(t)
    \Biggr] \label{eqn:lem-dp-ps-low-sens-indep} \\
    &\quad \leq e^{\epsilon} \max_{t \in \irange{n_0, n}} t \Delta_t \Pr(\Pi \in \vec{\Pi}(\alpha)) 
    + 2 \Delta^\star \tilde{\delta} \label{eqn:lem-dp-ps-low-sens-bounded}
  \end{align}
  where \eqref{eqn:lem-dp-ps-low-sens-same-dist} follows since $(\vec{X}, X')$ and $(\replace(\vec{X}, t, X'), X_{t})$ 
  have the same distribution, and
  \eqref{eqn:lem-dp-ps-low-sens-indep} follows since $X'$ is independent of $\Pi$ and 
  $\E_{X \sim \scrP}[v(q^\star, t^\star, \replace(\vec{X}, t, X))] = v(q^\star, t^\star, \vec{X}_{\irange{t - 1}})$.
  
  Subtracting $\sum_{\pi \in \vec{\Pi}(\alpha)} t^\star \Delta_{t^\star}$ 
  from both sides of inequality \eqref{eqn:lem-dp-ps-low-sens-bounded}, we have
  \begin{align}
    &\sum_{t = 1}^{n} \E_{\vec{X} \sim \scrP^n} \left[
    \sum_{\pi \in \vec{\Pi}(\alpha)} \Pr(\Pi = \pi \mid \vec{X}) \ind{t \leq t^\star} \left(
    v(q^\star, t^\star, \vec{X}_{\irange{t}}) - v(q^\star, t^\star, \vec{X}_{\irange{t - 1}}) 
    \right)
    \right] \\
    &\quad \leq \left(e^{\epsilon} \max_{t \in \irange{n_0, n}} t \Delta_t - \min_{t \in \irange{n_0,n}} t 
    \Delta_t\right) \Pr(\Pi \in \vec{\Pi}(\alpha)) 
    + 2 \Delta^\star \tilde{\delta}.
  \end{align}
  
  Now fix 
  $\alpha = \euler^{\epsilon} \max_{t \in \irange{n_0, n}} t \Delta_t - \min_{t' \in \irange{n_0, n}} t' \Delta_{t'} + 
  4 c n \max_{t'' \in \irange{n_0, n}} \Delta_{t''}$. 
  Suppose $\Pr(\abs{q^\star(\scrQ_\Pi^{t^\star}) - q^\star(\scrP^{t^\star})} > \alpha) > \frac{1}{c}$. 
  Then it must be that either $\Pr(q^\star(\scrQ_\Pi^{t^\star}) - q^\star(\scrP^{t^\star}) > \alpha) > 
  \frac{1}{2c}$ 
  or $\Pr(q^\star(\scrP^{t^\star}) - q^\star(\scrQ_\Pi^{t^\star}) > \alpha) > \frac{1}{2c}$. 
  Without loss of generality, assume 
  \begin{equation}
    \Pr(q^\star(\scrQ_\Pi^{t^\star}) - q^\star(\scrP^{t^\star}) > \alpha) = \Pr(\Pi \in \vec{\Pi}(\alpha)) > 
    \frac{1}{2c}.
  \end{equation}
  But this leads to a contradiction since 
  \begin{align}
    &\alpha \Pr(\Pi \in \vec{\Pi}(\alpha))  \\
    &< \sum_{\pi \in \vec{\Pi}(\alpha)} \Pr(\Pi = \pi) \left(
    q^\star(\scrQ_\pi^{t^\star}) - q^\star(\scrP^{t^\star})
    \right) \\
    &= \E_{\vec{X} \sim \scrP^n} \left[
    \sum_{\pi \in \vec{\Pi}(\alpha)} \Pr(\Pi = \pi \mid \vec{X}) \left(
    q^\star(\vec{X}_{\irange{t^\star}}) - q^\star(\scrP^{t^\star})
    \right)
    \right] \\
    &= \sum_{t = 1}^{n} \E_{\vec{X} \sim \scrP^n} \left[
    \sum_{\pi \in \vec{\Pi}(\alpha)} \Pr(\Pi = \pi \mid \vec{X}) \ind{t \leq t^\star} \left(
    v(q^\star, t^\star, \vec{X}_{\irange{t}}) - v(q^\star, t^\star, \vec{X}_{\irange{t - 1}})
    \right)
    \right] \\
    &\leq \left(
    \euler^{\epsilon} \max_{t \in \irange{n_0, n}} t \Delta_t - \min_{t' \in \irange{n_0,n}} t' \Delta_{t'} 
    \right) \Pr(\Pi \in \vec{\Pi}(\alpha)) + 2 \Delta^\star \tilde{\delta} \\
    &\leq \Pr(\Pi \in \vec{\Pi}(\alpha)) \left(
    \euler^{\epsilon} \max_{t \in \irange{n_0, n}} t \Delta_t - \min_{t' \in \irange{n_0, n}} t' \Delta_{t'} 
    + 4 c \Delta^\star \tilde{\delta}  
    \right)
  \end{align}
\end{proof}

\begin{lemma}[Lemma 21,~\citealp{jung2020new}] \zlabel{lem:dp}
  If $\sinteract(\blank; \blank, \mech)$ is $(\epsilon, \vec{\delta})$-differentially private, then for any event 
  $E \in \scrT$, any index $t \in \irange{n}$ and value $x \in \scrX$:
  \begin{equation}
    \Pr_{\vec{X} \sim \scrP^n, \Pi \sim \sinteract(\replace(\vec{X}, t, x))} [\Pi \in E] 
    \leq \euler^\epsilon \Pr_{\vec{X} \sim \scrP^n, \Pi \sim \sinteract(\vec{X})} [\Pi \in E] + \vec{\delta}(t)
  \end{equation}
  where $\replace(\vec{X}, t, x)$ is the dataset obtained from $\vec{X}$ by replacing the $t$-th data point with $x$.
\end{lemma}
\begin{proof}
  This follows from expanding the definitions
  \begin{align}
    \Pr_{\vec{X} \sim \scrP^n, \Pi \sim \sinteract(\replace(\vec{X}, t, x))} [\Pi \in E] 
    &= \sum_{\vec{x} \in \scrX^n} 
    \Pr_{\vec{X} \sim \scrP^n} [\vec{X} = \vec{x}] \Pr_{\Pi \sim \sinteract(\replace(\vec{x}, t, x))} [\Pi \in E] \\
    &\leq \sum_{\vec{x} \in \scrX^n} 
    \Pr_{\vec{X} \sim \scrP^n} [\vec{X} = \vec{x}] 
    \left(\euler^{\epsilon} \Pr_{\Pi \sim \sinteract(\vec{x})} [\Pi \in E]  + \vec{\delta}(t) \right) \label{eqn:lem21-dp} \\
    &= \euler^\epsilon \Pr_{\vec{X} \sim \scrP^n, \Pi \sim \sinteract(\vec{X})} [\Pi \in E] + \vec{\delta}(t)
  \end{align}
  where \eqref{eqn:lem21-dp} follows from the definition of differential privacy.
\end{proof}

\section{Proofs for \texorpdfstring{\zcref{sec:application}}{Section~\ref{sec:application}}} 
\zlabel{app:proofs-application}

\subsection{Generalization Guarantees Assuming Fixed Privacy Parameters}

\lemgaussmechdp*
\begin{proof}
  We begin by bounding the privacy of the interaction using zero-concentrated differential privacy (zCDP), as 
  defined in \zcref{def:zcdp}.
  When releasing an estimate to a single statistical query in round $\tau$, the Gaussian mechanism satisfies 
  $\vec{\rho}$-zCDP with $\vec{\rho}(t) = \ind{t \leq \tau} / 2 \sigma_\tau^2 \tau^2$, where we have used 
  $1 / \tau^2$ as an upper bound on the sensitivity of the query \citep[Proposition~1.6]{bun2016concentrated}.
  By post-processing (see \zcref{thm:zcdp-post-processing}), clipping the output of the Gaussian mechanism does not 
  accrue an additional privacy loss, nor does the analyst's choice of the next query conditioned on the previous 
  releases. 
  Now by advanced composition (see \zcref{thm:zcdp-composition}), the interaction satisfies $\vec{\rho}$-zCDP with 
  $\vec{\rho}(t) = \sum_{\tau = n_0}^{n} k_\tau \ind{t \leq \tau} / 2 \sigma_\tau^2 \tau^2$. 
  We note that in order to apply this theorem, we have relied on the fact that the privacy parameters are non-adaptive, 
  which is true when the $k_t$'s are fixed in advance.  
  Finally we convert from $\vec{\rho}$-zCDP to $(\epsilon, \vec{\delta})$-DP using 
  \zcref{cor:zcdp-approx-dp-conversion}.
\end{proof}

\lemgaussmechacc*
\begin{proof}
  Let $k_t$ denote the number of queries asked in round $t$.
  Let $j \in \{1, \ldots, k_t\}$ index the queries asked in round $t$ and 
  $Z_{t,j} \overset{\mathrm{iid.}}{\sim} \operatorname{Normal}(0, \sigma_t)$ denote the Gaussian noise added by the 
  mechanism to the $j$-th query in round $t$. 
  
  We begin with the observation that
  \begin{align}
    \Pr \Bigl( \max_{t \in \irange{n_0, n}} \max_{j \in \irange{k_t}} (-Z_{t,j} - \alpha_t) \geq 0 \Bigr)
    &= 1 - \Pr\Bigl( \max_{t \in \irange{n_0, n}} \max_{j \in \irange{k_t}} (- Z_{t,j} - \alpha_t) < 0 \Bigr) \\
    &= 1 - \prod_{t = n_0}^{n} \prod_{j = 1}^{k_t} \Pr(Z_{t,j} \geq -\alpha_t) \\
    &= 1 - \prod_{t = n_0}^{n} \prod_{j = 1}^{k_t} \left( 
    1 - \frac{1}{2} \erfc\left(\frac{\alpha_t}{\sqrt{2} \sigma_t}\right) 
    \right) \\
    &= 1 - \left(1 - \frac{1}{2} \erfc\left(\frac{\alpha_t}{\sqrt{2} \sigma_t}\right)
    \right)^k, \label{eqn:lem-gaus-acc-inter}
  \end{align}
  where the last equality follows from the fact that $\alpha_t / \sigma_t$ is constant with respect to 
  $t$. 
  By symmetry, this equality also holds under the replacement $Z_{t,j} \to -Z_{t,j}$.
  
  Conditioning on the number of queries $k_t$ asked in each round $t$, we have with respect to the joint distribution 
  on the dataset $\vec{X}$ and transcript $\Pi$ that
  \begin{align}
    &\Pr \Biggl( 
    \bigcup_{t = n_0}^{n} \bigcup_{(q,r) \in \Pi_t} \left\{ \abs{q(\vec{X}_{\irange{t}}) - r} \geq \alpha_t \right\} 
    \Bigg| 
    \bigcap_{t = n_0}^{n} \{\abs{\Pi_t} = k_t\}
    \Biggr) \\
    &\quad = \Pr \Bigl(
    \max_{t \in \irange{n_0, n}} \max_{j \in \irange{k_t}} \left\{ \abs{q_{t,j}(\vec{X}_{\irange{t}}) - 
    \clamp_{[0,1]}(q_{t,j}(\vec{X}_{\irange{t}}) + Z_{t,j})} - 
    \alpha_t 
    \right\} \geq 0 
    \Bigr) \\
    &\quad \leq \Pr \Bigl(
    \max_{t \in \irange{n_0, n}} \max_{j \in \irange{k_t}} \left\{ \abs{Z_{t,j}} - \alpha_t \right\} \geq 0 
    \Bigr) \\
    &\quad \leq \Pr \Bigl(
    \max_{t \in \irange{n_0, n}} \max_{j \in \irange{k_t}} \left\{ -Z_{t,j} - \alpha_t \right\} \geq 0 
    \Bigr) + \Pr \Bigl(
    \max_{t \in \irange{n_0, n}} \max_{j \in \irange{k_t}} \left\{ Z_{t,j} - \alpha_t \right\} \geq 0 
    \Bigr) \\
    &\quad = 2 \left(
    1 - \left(1 - \frac{1}{2} \erfc \mathopen{} \left(\frac{\alpha_t}{\sqrt{2} \sigma_t}\right)\right)^{k}
    \right), \label{eqn:lem-gaus-acc}
  \end{align}
  where the last line follows from \eqref{eqn:lem-gaus-acc-inter}. 
  Note that the bound only depends on the total number of queries $k$, not the number of queries asked at each 
  time step $k_t$, thus it serves as a bound on the probability conditioned on the event $\sum_{t = n_0}^{n} 
  \abs{\Pi_t} = k$. 
  The result follows by setting \eqref{eqn:lem-gaus-acc} equal to $\beta$ and solving for $\alpha_t / \sqrt{2} \sigma_t$.
\end{proof}

\thmgausmechdistacc*
\begin{proof}
  Combining \zcref{lem:gauss-mech-dp,lem:gauss-mech-acc} and \zcref{thm:transfer-dp} implies the 
  clipped Gaussian mechanism is $(\alpha', \beta')$-distributionally accurate 
  for $\beta' = \beta / d + 1 / c$ and 
  \begin{equation}
    \alpha' = \sqrt{2} \sigma \erfc^{-1} \mathopen{} \left(\frac{\beta}{k}\right) + \euler^\epsilon - 1 
      + \frac{2 c \sum_{\tau = 1}^{n} \vec{\delta}(\tau)}{n_0} + d
  \end{equation}
  for any $c, d > 0$ and $0 < \beta < 1$. 
  We select the free parameters to minimize $\alpha'$. 
  Eliminating $c$ using the constraint $\beta' = \beta / d + 1 / c$ and minimizing with respect to $d$ 
  analytically gives:
  \begin{equation}
    \alpha' = \sqrt{2} \sigma \erfc^{-1} \mathopen{} \left(\frac{\beta}{k}\right) + \euler^{\epsilon} - 1 
      + \frac{2 \sum_{\tau = 1}^{n} \vec{\delta}(\tau)}{n_0 \beta'} + \frac{\beta}{\beta'} + \frac{2}{\beta'} \sqrt{\frac{2 \beta \sum_{\tau = 1}^{n} \vec{\delta}(\tau)}{n_0}}
  \end{equation}
  Minimizing this expression with respect to the remaining free parameters gives the required result.
\end{proof}

\subsection{Generalization Guarantees Assuming Adaptive Privacy Parameters} \label{app:application-fully-adaptive}

In this appendix, we instantiate generalization guarantees for the clipped Gaussian mechanism where we allow the 
analyst to \emph{adaptively} select how many queries to ask in each round. 
This is in contrast with the results of \zcref{sec:application-guarantee}, where we assume the number of queries asked 
in each round is \emph{fixed} before interacting with the data. 
To support a fully adaptive analyst, we rely on a \emph{privacy filter}, which provides differential privacy guarantees 
when both the mechanisms and privacy parameters are selected adaptively. 
Generally speaking, a privacy filter is an algorithm that continually monitors the privacy parameters of mechanisms as 
they are composed. 
At any point, it can force the composition to terminate to ensure it satisfies a pre-specified level of privacy. 

We use a privacy filter proposed by \citet{whitehouse2023fully} for approximate zero-concentrated differential 
privacy (zCDP).
It achieves the same rates as advanced composition~\citep{bun2016concentrated} assuming the privacy level is 
specified upfront. 
The filter does not support non-uniform privacy accounting---where the privacy loss is estimated separately for 
each data point \slash individual. 
We therefore resort to uniform accounting, which results in looser generalization bounds. 
We expect the filter could be adapted to support non-uniform privacy accounting, in a similar way to the R\'enyi 
filter proposed by \citet{feldman2021individual}. 

To incorporate \citeauthor{whitehouse2023fully}'s privacy filter for an adaptive analysis using the clipped Gaussian 
mechanism, we require additional monitoring as shown in \zcref{alg:filter}. 
Lines~\ref{alg-line:init-rho} and~\ref{alg-line:update-rho} track the running privacy level $\bar{\rho}$ using the same 
additive rule as advanced composition for zCDP. 
If releasing a response to a query would not exceed the target privacy level (line~\ref{alg-line:privacy-test}), 
then a response to the query is released (line~\ref{alg-line:release}) and the analysis continues, 
otherwise the mechanism is terminated (line~\ref{alg-line:terminate}).

\begin{algorithm}[H]
  \caption{Composition of Clipped Gaussian Mechanisms with zCDP Privacy Filter}
  \zlabel{alg:filter}
  \begin{algorithmic}[1]
    \Statex \textbf{Input:} target privacy level $\rho$
    \State $\bar{\rho} \gets 0$ \label{alg-line:init-rho}
    \For {Round $t \in \irange{n_0,n}$}
      \While {Analyst has another query $q$}
        \State $\bar{\rho} \gets \bar{\rho} + \frac{1}{2\sigma_t^2 t^2}$ \label{alg-line:update-rho}
        \If {$\bar{\rho} \leq \rho$} \label{alg-line:privacy-test}
          \State Sample noise: $z \sim \normal(0, \sigma_t^2)$ 
          \State Return response to query: $r \gets \clamp_{[0,1]}(q(\vec{x}_{\irange{t}}) + z)$ 
          \label{alg-line:release}
        \Else 
          \State TERMINATE \label{alg-line:terminate}
        \EndIf
      \EndWhile
    \EndFor
  \end{algorithmic}
\end{algorithm}

\begin{lemma} \zlabel{lem:gauss-mech-dp-adaptive}
  Consider an interaction $\sinteract(\cdot; \cdot, \mech)$ where $\mech$ is a composition of clipped Gaussian 
  mechanisms tracked by a privacy filter that satisfies $\rho$-zCDP, as described in \zcref{alg:filter}. 
  Then the interaction satisfies $\rho$-zCDP, with the possibility that it may terminate early. 
  It also satisfies $(\epsilon, \delta)$-DP for any $\epsilon \geq 0$ with 
  $\delta \leq \inf_{\gamma \in (1, \infty)} \euler^{(\gamma - 1)(\gamma \rho - \epsilon)} (1 - \gamma^{-1})^{\gamma} / (\gamma - 1)$.
\end{lemma}
\begin{proof}
  We recall that the Gaussian mechanism satisfies $1 / (2 \sigma_t^2 t)$-zCDP for a statistical query with 
  sensitivity $\Delta \leq 1 / t$ \citep[Proposition~1.6]{bun2016concentrated}. 
  By post-processing for zCDP, clipping the result of the Gaussian mechanism and selecting a query as a function 
  of the result (and past results) does not accrue an additional cost to privacy. 
  Hence, by fully adaptive composition for the zCDP privacy filter \citep[Theorem~1]{whitehouse2023fully}, the 
  interaction satisfies $\rho$-zCDP.
  Finally, we convert from $\rho$-zCDP to $(\epsilon, \delta)$-DP to obtain the final 
  result~\citep[Corollary~13]{canonne2020discrete}.
\end{proof}

Combining \zcref{lem:gauss-mech-dp-adaptive} with \zcref{lem:gauss-mech-acc} and \zcref{thm:transfer-dp} 
yields the following generalization guarantee. 
We note that this essentially matches the guarantee for the more constrained setting where the number of queries in 
each round is fixed upfront (\zcref{thm:gaus-mech-dist-acc}).
\begin{theorem} \zlabel{thm:gaus-mech-dist-acc-adaptive}
  Consider an interaction $\sinteract(\cdot; \cdot, \mech)$ where $\mech$ is a composition of clipped Gaussian 
  mechanisms tracked by a zCDP privacy filter that satisfies $(\epsilon, \delta)$-DP.
  Suppose $\mech$ answers $k$ statistical queries without terminating and assume $\sigma_t = \sigma > 0$.
  Then $\mech$ is $(\alpha', \beta')$-distributionally accurate for $k$ statistical queries with
  \begin{align}
    \alpha' = \min_{(\sigma, \beta, \epsilon) \in \Theta}
    \sqrt{2} \sigma \erfc^{-1} \mathopen{} \left(\frac{\beta}{k}\right) + \euler^{\epsilon} - 1 
    + \frac{2 n \delta}{n_0 \beta'} + \frac{\beta}{\beta'} + \frac{2}{\beta'} \sqrt{\frac{2 \beta n \delta}{n_0}}
  \end{align}
  where $\Theta$ is defined in \zcref{thm:gaus-mech-dist-acc} and $\delta$ is defined in 
  \zcref{lem:gauss-mech-dp-adaptive}.
\end{theorem}
\begin{proof}
  The proof is similar to the proof of \zcref{thm:gaus-mech-dist-acc}.
\end{proof}

\subsection{Generalization Guarantees using the Static Bound of \texorpdfstring{\citet{jung2020new}}{Jung et al. (2020)}} \zlabel{app:jlnrss-bound}

In this appendix, we derive a generalization guarantee for the setting considered in \zcref{sec:application-empirical} 
by composing a bound for static data due to \citet{jung2020new}. 
In doing so, we will obtain the bound that is labeled \textbf{JLNRSS} in the empirical results of 
\zcref{sec:application-empirical}.
For completeness, we begin by restating \citeauthor{jung2020new}'s result below, with some minor changes in notation.
In particular, we make the dependence of the bound on the dataset size $n$ and number of queries $k$ explicit, since 
we will apply the bound to multiple batches, each with a different value of $n$ and $k$.
\begin{theorem} \zlabel{thm:gauss-mech-dist-acc-jlnrss-static}
    Consider ADA on a static dataset of size $n$---i.e., an instantiation of \zcref{alg:interaction} with 
    $n_0 = n - 1$. 
    The clipped Gaussian mechanism can be used to answer $k$ statistical queries while satisfying 
    $(\alpha(n, k, \beta), \beta)$-distributional accuracy for any $\beta \in (0, 1)$ with
    \begin{equation}
      \alpha(n, k, \beta) = \min_{\sigma, \delta > 0} \left\{
        \sqrt{2} \sigma \erfc^{-1} \mathopen{} \left(\frac{\delta}{k}\right) 
        + \exp \left( \frac{k}{2 n^2 \sigma^2} + \sqrt{\frac{2 k}{n^2 \sigma^2} \log \left( \frac{\sqrt{\pi k}}{\sqrt{2} n \sigma \delta} \right)} \right) 
        - 1 + 6 \frac{\delta}{\beta}
        \right\}.
    \end{equation}
\end{theorem}
We note that an improved bound can be obtained by (1)~not constraining $\beta' = \delta$ and $c = d$ in the proof 
and (2)~using a tighter conversion from zCDP to approximate DP based on \citet{canonne2020discrete} in place of 
\citet{bun2016concentrated}. 
This improved bound is used in the approach labeled \textbf{JLNRSS+} in \zcref{sec:application-empirical}.

We now turn to the batched setting described in \zcref{sec:application-empirical}. 
Recall that the data arrives in $b$ batches, which we index by $\ell \in \irange{b}$. 
We let $n_\ell$ denote the size of the $\ell$-th data batch, and $k_\ell$ denote the number of queries asked by the 
analyst after the batch arrives. 
Since the $\ell$-th data batch is static while the analyst asks the $k_\ell$ queries, we can bound the worst-case 
distributional accuracy for the entire analysis by applying a static bound on each batch, and taking the union bound. 
\begin{proposition} \zlabel{prop:gauss-mech-dist-acc-jlnrss}
  Consider ADA on growing data in the batched setting. 
  The clipped Gaussian mechanism can be used to answer $k$ statistical queries while satisfying 
  $(\alpha', \beta')$-distributional accuracy for any $\beta' \in (0, 1)$ with 
  $\alpha' = \max_{\ell \in \irange{b}} \alpha(n_\ell, k_\ell, \beta'/b)$, where $\alpha(\cdot, \cdot, \cdot)$ is as 
  defined in \zcref{thm:gauss-mech-dist-acc-jlnrss-static}.
\end{proposition}
\begin{proof}
  Let $\vec{X} \sim \scrP^n$ and $\Pi \sim \sinteract(\vec{X}; \analyst, \mech)$. 
  It is convenient to refer to the queries and responses in the transcript $\Pi$ using two indices:
  the batch $\ell \in \irange{b}$ when the query was asked and the number $j \in \irange{k_\ell}$ of the query within 
  batch $\ell$. 
  Adopting these indices, we say that the clipped Gaussian mechanism is $(\alpha', \beta')$-distributionally accurate 
  if with probability $1 - \beta'$ over the randomness in $\vec{X}$ and $\Pi$, we have that 
  $\max_{\ell \in \irange{b}} \max_{j \in \irange{k_\ell}} \abs{R_{\ell,j} - Q_{\ell,j}(\scrP^{n_\ell})} \geq \alpha'$
  for any data distribution $\scrP$ and analyst $\analyst$. 

  Now we obtain an upper bound on the probability of interest by replacing the error $\alpha'$ with a lower bound for 
  each batch $\ell$, and finally applying the union bound:
  \begin{align*}
    &\Pr_{\vec{X} \sim \scrP^{n}, \Pi \sim \sinteract(\vec{X})} \Biggl(
        \max_{\ell \in \irange{b}} \max_{j \in \irange{k_\ell}} \abs{R_{\ell,j} - Q_{\ell,j}(\scrP^{n_\ell})} \geq \max_{\ell' \in \irange{b}} \alpha_{n_{\ell'}, k_{\ell'}}
      \Biggr) \\
    &= \Pr_{\vec{X} \sim \scrP^{n}, \Pi \sim \sinteract(\vec{X})} \Biggl(
        \max_{\ell \in \irange{b}} \left\{\max_{j \in \irange{k_\ell}} \abs{R_{\ell,j} - Q_{\ell,j}(\scrP^{n_\ell})} - \max_{\ell' \in \irange{b}} \alpha_{n_{\ell'}, k_{\ell'}} \right\} \geq 0
      \Biggr) \\
    &\leq \Pr_{\vec{X} \sim \scrP^{n}, \Pi \sim \sinteract(\vec{X})} \Biggl(
        \max_{\ell \in \irange{b}} \left\{\max_{j \in \irange{k_\ell}} \abs{R_{\ell,j} - Q_{\ell,j}(\scrP^{n_\ell})} - \alpha_{n_\ell, k_\ell} \right\} \geq 0 
      \Biggr) \\
    &= \Pr_{\vec{X} \sim \scrP^{n}, \Pi \sim \sinteract(\vec{X})} \Biggl(
        \bigcup_{\ell = 1}^{b} \left\{ \max_{j \in \irange{k_\ell}} \abs{R_{\ell,j} - Q_{\ell,j}(\scrP^{n_\ell})} \geq \alpha_{n_\ell, k_\ell} \right\} 
      \Biggr) \\
    &\leq b (\beta'/b).
  \end{align*}
\end{proof}

\subsection{Generalization Guarantees for Data Splitting} \zlabel{app:sample-split}

Data splitting is a simple method for mitigating the risk of overfitting when conducting an adaptive analysis. 
It involves randomly splitting an i.i.d.\ dataset into disjoint chunks, and using a fresh chunk whenever a step 
in the analysis depends on existing data. 
Data splitting has been used as a baseline in prior work for static data~\citep{bassily2016algorithmic,jung2020new,
rogers2020guaranteed} and can be adapted for growing data by splitting the data into chunks as data points arrive. 
There is a limitation in the growing setting: it may be necessary to delay a step in the analysis if sufficient 
data has not yet arrived to create a fresh chunk of the desired size. 
This is not a limitation of our approach (\zcref{alg:interaction}) which can respond to queries without delay at any 
time. 
For completeness, we provide a high probability worst-case generalization bound for data splitting below. 

\begin{proposition}
  Data splitting is $(\alpha, \beta)$-distributionally accurate when used to answer $k$ adaptive statistical queries 
  for any $\alpha, \beta \geq 0$ such that $\sum_{j = 1}^{k} \euler^{- 2 b_j \alpha^2} = \frac{\beta}{2}$, where $n_j$ 
  is the (predetermined) size of the split used to answer the $j$-th query.
  In particular, if a dataset of size $n$ is split evenly across the $k$ queries so that $n_j = n / k \in \naturals$ 
  then $n = \frac{k}{2 \alpha^2} \log \frac{2k}{\beta}$.
\end{proposition}
\begin{proof}
  We write $(q_j, r_j)$ to refer to the $j$-th query-estimate pair in the flattened transcript.
  We also define $m_j = \sum_{j' = 1}^{j - 1} n_{j'}$ to be the number of data points used by the mechanism prior to 
  answering the $j$-th query.
  By the union bound and Hoeffding's bound we have
  \begin{align}
    &\Pr_{\vec{X} \sim \scrP^n, \Pi \sim \sinteract(\vec{X})} \Biggl( 
      \bigcup_{j = 1}^{k} \{\abs{r_j - q_j(\scrP)} \geq \alpha\} 
    \Biggr) \\ 
    &\quad = \Pr_{\vec{X} \sim \scrP^n, \Pi \sim \sinteract(\vec{X})} \Biggl( 
      \bigcup_{j = 1}^{k} \Biggl\{
        \abs*{\frac{1}{n_j} \sum_{t = m_j}^{m_j + n_j} q_j(X_t) - \E_{X \sim \scrP}[q_j(X)]} \geq \alpha 
      \Biggr\} 
    \Biggr) \\
    &\quad \leq \sum_{j = 1}^{k} \Pr_{\vec{X} \sim \scrP^n, \Pi \sim \sinteract(\vec{X})} \Biggl( 
      \abs*{\frac{1}{n_j} \sum_{t = m_j}^{m_j + n_j} \left\{q_j(X_t) - \E_{X \sim \scrP}[q_j(X)]\right\}} \geq 
      \alpha 
    \Biggr) \\
    &\quad \leq \sum_{j = 1}^{k} 2 \euler^{-2 b_j \alpha^2}
  \end{align}
\end{proof}

\section{Results for Non-uniform Privacy Parameters} \zlabel{app:non-uniform-privacy}

In this appendix, we prove key results for differential privacy with non-uniform privacy parameters. 
Although many of the results are identical to the uniform case, we were unable to find proofs in the literature. 
We begin by extending the definitions of approximate differential privacy (approx.\ DP) and zero-concentrated 
differential privacy (zCDP) to the non-uniform setting. 
Then we prove composition and post-processing for these definitions in \zcref{app:non-uniform-privacy-comp-pp}, and 
a conversion theorem from zCDP to approx.\ DP in \zcref{app:non-uniform-conversion}. 

Informally, differential privacy (DP) is a bound on how distinguishable the outputs of a randomized algorithm will be 
when run on two neighboring datasets. 
Ordinarily, the bound on distinguishability holds uniformly over all neighboring datasets, which means the level of 
privacy is the same for all records in the dataset. 
However, in some scenarios it may be tolerable for the privacy guarantee to vary non-uniformly over records---e.g., 
where individuals have different privacy preferences, or where privacy is permitted to decay as records age. 
Non-uniform privacy definitions have been studied using pure DP as a foundation, which is known as
\emph{personalized DP}~\citep{ebadi2015differential,jorgensen2015conservative}. 
Below, we extend this idea to approximate DP, by upgrading the $\delta$ parameter from a constant to a function 
$\vec{\delta}(\cdot)$ that varies for each record index.

\begin{definition}[Approximate DP] \zlabel{def:approx-dp}
  Let $\epsilon \geq 0$ and $\vec{\delta} \colon \irange{n} \to [0, 1]$.
  A randomized mechanism $\mech \colon \scrX^n \to \scrY$ satisfies $(\epsilon, \vec{\delta})$-differential privacy or 
  $(\epsilon, \vec{\delta})$-DP for short, if for all indices $i \in \irange{n}$, all pairs of neighboring 
  datasets $(\vec{x}, \vec{x}') \in \scrN_i$ differing on the $i$-th entry, and all measurable events 
  $E \subseteq \scrY$: 
  \begin{equation}
    \Pr \mathopen{} \left(\mech(\vec{x}) \in E\right)
      \leq \euler^{\epsilon} \Pr \mathopen{} \left(\mech(\vec{x}') \in E\right) + \vec{\delta}(i).
  \end{equation}
\end{definition}

Note that this definition depends on $\scrN_i$, the set of neighboring datasets that differ on the $i$-th record. 
One could consider \emph{unbounded} neighboring datasets, in which case $\scrN_i$ would consist of pairs of 
datasets $(\vec{x}, \vec{x}')$ such that $\vec{x}'$ can be obtained from $\vec{x}$ by adding or removing the $i$-th 
record. 
Alternatively, one could consider \emph{bounded} neighboring datasets, where the pairs of datasets 
$(\vec{x}, \vec{x}')$ in $\scrN_i$ are such that $\vec{x}'$ can be obtained from $\vec{x}$ by changing the $i$-th 
record.

Next, we define a non-uniform variant of zero-concentrated differential privacy (zCDP).
However, we must first define the privacy loss distribution, since it is used to measure indistinguishability 
for zCDP.
\begin{definition}[Privacy loss distribution] \label{def:priv-loss-dist}
  Let $P$ and $Q$ be probability distributions on $\scrY$.\footnote{%
    For simplicity we assume $\scrY$ is discrete, so that we don't have to worry about measure theory.
  } 
  Define $f_{P \| Q} \colon \scrY \to \reals$ such that $f_{P \| Q}(y) = \log (P(y) / Q(y))$. 
  The privacy loss random variable is given by $Z = f_{P \| Q}(Y)$ for $Y \gets P$. 
  The distribution of $Z$ is denoted by $\privloss{P}{Q}$.
\end{definition}
The standard definition of zCDP bounds the moment generating function of the privacy loss random variable 
$Z = \privloss{\mech(\vec{x})}{\mech(\vec{x}')}$ uniformly over pairs of neighboring datasets, in terms of a scalar 
$\rho$~\citep{bun2016concentrated}. 
We upgrade the scalar $\rho$ to a function $\vec{\rho}(\cdot)$ that varies for each record index. 
\begin{definition}[Zero-concentrated DP] \zlabel{def:zcdp}
  Let $\vec{\rho} \colon \irange{n} \to [0, \infty)$. 
  A randomized mechanism $\mech \colon \scrX^n \to \scrY$ satisfies $\vec{\rho}$-zero-concentrated differential privacy 
  or $\vec{\rho}$-zCDP for short, if for all indices $i \in \irange{n}$ and all pairs of neighboring datasets 
  $(\vec{x}, \vec{x}') \in \scrN_i$ that differ on the $i$-th entry, the privacy loss distribution 
  $\privloss{\mech(\vec{x})}{\mech(\vec{x}')}$ is well-defined and 
  \begin{align}
    \forall \tau \geq 0,  \E_{Z \gets \privloss{\mech(\vec{x})}{\mech(\vec{x}')}} \left(\exp(\tau Z)\right)
      \leq \exp(\tau (\tau + 1) \vec{\rho}(i)).
  \end{align}
\end{definition}

\subsection{Composition and post-processing} \zlabel{app:non-uniform-privacy-comp-pp}

We need a composition theorem for $\vec{\rho}$-zCDP to analyze the privacy of successive applications of the clipped 
Gaussian mechanism in \zcref{lem:gauss-mech-dp}.
We show that $\vec{\rho}$-zCDP composes in the obvious way: by adding the $\vec{\rho}$ privacy parameters pointwise.
\begin{theorem}[Composition for $\vec{\rho}$-zCDP] \zlabel{thm:zcdp-composition}
  Let randomized mechanism $\mech_1 : \scrX^\star \to \scrY_1$ satisfy $\vec{\rho}_1$-zCDP. 
  Let $\mech_2 : \scrX^\star \times \scrY_1 \to \scrY_2$  be such that, for all $y \in \scrY_1$, the restriction 
  $\mech_2(\cdot, y) : \scrX^\star \to \scrY_2$ satisfies $\vec{\rho}_2$-zCDP. 
  Define $\mech : \scrX^\star \to \scrY_1 \times \scrY_2$ such that 
  $Y_1 \gets \mech_1(\vec{x})$, $Y_2 | Y_1 \gets \mech_2(\vec{x}, Y_1)$ and $\mech(\vec{x}) = (Y_1, Y_2)$. 
  Then $\mech$ satisfies $\vec{\rho}$-zCDP with $\vec{\rho}(i) = \vec{\rho}_1(i) + \vec{\rho}_2(i)$ for all 
  $i \in \irange{n}$.
\end{theorem}
\begin{proof}
  We adapt the proof of \citet{steinke2022composition}.
  Fix $i \in \irange{n}$ and neighboring datasets $(\vec{x}, \vec{x}') \in \scrN_i$ that differ on the $i$-th entry. 
  Fix $\tau \geq 0$. 
  Let $Z \gets \privloss{\mech(\vec{x})}{\mech(\vec{x}')}$ where we conceal the dependence on $i$.
  We must prove $\E(\exp(\tau Z)) = \exp(\tau (\tau + 1) (\rho_1(i) + \rho_2(i)))$.

  The privacy loss distribution for $\mech$ can be decomposed as follows\footnote{%
    We again assume that $\scrY$ is discrete for simplicity, however the result holds more generally.
  }:
  \begin{align}
    f_{\mech(\vec{x}) \| \mech(\vec{x}')}(y_1, y_2) 
    &= \log \frac{\Pr (\mech(\vec{x}) = (y_1, y_2))}{\Pr (\mech(\vec{x}') = (y_1, y_2))} \\
    &= \log \frac{\Pr (\mech_1(\vec{x}) = y_1) \Pr (\mech_2(\vec{x}, y_1) = y_2)}
      {\Pr(\mech_1(\vec{x}') = y_1) \Pr(\mech_2(\vec{x}', y_1) = y_2)} \\
    &= \log \frac{\Pr(\mech_1(\vec{x}) = y_1)}{\Pr(\mech_1(\vec{x}') = y_1)} 
      + \log \frac{\Pr(\mech_2(\vec{x}, y_1) = y_2)}{\Pr(\mech_2(\vec{x}', y_1) = y_2)} \\
    &= f_{\mech_1(\vec{x}) \| \mech_1(\vec{x}')}(y_1) + f_{\mech_2(\vec{x}, y_1) \| \mech_2(\vec{x}', y_1)}(y_2).
  \end{align}
  Hence
  \begin{align}
    &\E_{Z \gets \privloss{\mech(\vec{x})}{\mech(\vec{x}')}} \mathopen{} \left(\exp(\tau Z)\right) \\
    &= \E_{Y_1 \gets \mech_1(\vec{x}), Y_2 \gets \mech_2(\vec{x}, Y_1)} \mathopen{} \left(
      \exp(\tau f_{\mech(\vec{x}) \| \mech(\vec{x}')}(Y_1, Y_2))
      \right) \\
    &= \E_{Y_1 \gets \mech_1(\vec{x})} \mathopen{} \left(\exp(\tau f_{\mech_1(\vec{x}) \| \mech_1(\vec{x}')}(Y_1)) 
      \E_{Y_2 \gets \mech_2(\vec{x}, Y_1)} (\exp(\tau f_{\mech_2(\vec{x}) \| \mech_2(\vec{x}')}(Y_2)))
    \right) \\
    &\leq \E_{Y_1 \gets \mech_1(\vec{x})} \mathopen{} \left(
      \exp(\tau f_{\mech_1(\vec{x}) \| \mech_1(\vec{x}')}(Y_1)) \sup_{y_1 \in \scrY_1}
        \E_{Y_2 \gets \mech_2(\vec{x}, y_1)} (\exp(\tau f_{\mech_2(\vec{x}) \| \mech_2(\vec{x}')}(Y_2))) 
    \right) \\
    &= \E_{Z_1 \gets \privloss{\mech_1(\vec{x})}{\mech_1(\vec{x}')}} \mathopen{} \left(\exp(\tau Z_1)\right) \cdot 
      \sup_{y_1 \in \scrY_1} 
        \E_{Z_2 \gets \privloss{\mech_2(\vec{x}, y_1)}{\mech_2(\vec{x}', y_1)}} \mathopen{} \left(\exp(\tau Z_2)\right) \\
    &\leq \exp(\tau(\tau + 1) \rho_1(i)) \cdot \exp(\tau(\tau + 1) \rho_2(i)) \\
    &= \exp(\tau(\tau + 1) (\rho_1(i) + \rho_2(i)))
  \end{align}
  as required.
\end{proof}

We require post-processing of $\vec{\rho}$-zCDP to perform privacy accounting of the entire interaction between 
the analyst and mechanism in \zref{lem:gauss-mech-dp}. 
In essence, we model the adversary as a post-processing operation applied to previous responses from the clipped 
Gaussian mechanism, which selects the next query. 
We demonstrate below that post-processing holds in the non-uniform setting.
\begin{theorem}[Post-processing for $\vec{\rho}$-zCDP] \zlabel{thm:zcdp-post-processing}
  Let $\mech_0 \colon \scrX^n \to \scrY$ be a randomized mechanism that satisfies $\vec{\rho}$-zCDP and let 
  $f \colon \scrX^n \to \scrZ$ be an arbitrary randomized mapping. 
  Define the post-processed mechanism $\mech \colon \scrX^n \to \scrZ$ such that $\mech(\vec{x}) = f(\mech_0(\vec{x}))$.
  Then $\mech$ also satisfies $\vec{\rho}$-zCDP.
\end{theorem}
\begin{proof}
  Fix $i \in \irange{n}$ and neighboring datasets $(\vec{x}, \vec{x}') \in \scrN_i$ that differ on the $i$-th entry. 
  Fix $\tau \geq 0$.
  Applying Lemma~20 of \citet{steinke2022composition}, we have
  \begin{align}
    \E_{Z \gets \privloss{\mech(\vec{x})}{\mech(\vec{x}')}} \mathopen{} \left(\exp(\tau Z)\right)
    \leq \E_{Z_0 \gets \privloss{\mech_0(\vec{x})}{\mech_0(\vec{x}')}} \mathopen{} \left(\exp(\tau Z_0)\right).
  \end{align}
  Now the right-hand side is $\leq \exp(\tau(\tau + 1) \vec{\rho}(i))$ since $\mech_0$ is $\vec{\rho}$-zCDP, as required.
\end{proof}

We also show that post-processing holds for non-uniform approximate DP. 
This result is used to prove generalization guarantees for minimization queries in \zcref{thm:transfer-min-q}, which 
can be regarded as post-processed low sensitivity queries.
\begin{theorem}[Post-processing for $(\epsilon, \vec{\delta})$-DP] \zlabel{thm:approx-dp-post-processing}
  Let $\mech_0 \colon \scrX^n \to \scrY$ be a randomized mechanism that satisfies $(\epsilon, \vec{\delta})$-DP 
  and let $f \colon \scrX^n \to \scrZ$ be an arbitrary randomized mapping. 
  Define the post-processed mechanism $\mech \colon \scrX^n \to \scrZ$ such that $\mech(\vec{x}) = f(\mech_0(\vec{x}))$.
  Then $\mech$ also satisfies $(\epsilon, \vec{\delta})$-DP.
\end{theorem}
\begin{proof}
  We adapt the proof of \citet{dwork2014algorithmic}.
  First consider a deterministic mapping $f$. 
  Fix $i \in \irange{n}$ and neighboring datasets $(\vec{x}, \vec{x}') \in \scrN_i$ that differ on the $i$-th entry. 
  Fix any event $E \subseteq \scrO$ and let $G = \{o \in \scrO : f(o) \in E\}$.
  We then have 
  \begin{align}
    \Pr(\mech(\vec{x}) \in E) &= \Pr(\mech_0(\vec{x}) \in G) \\
    &\leq \euler^{\epsilon} \Pr(\mech_0(\vec{x}') \in G) + \vec{\delta}(i) \\
    &= \euler^{\epsilon} \Pr(\mech(\vec{x}') \in E) + \vec{\delta}(i),
  \end{align}
  which completes the proof for a deterministic mapping. 
  To extend the result to a random mapping, we can write $f$ as a mixture of deterministic mappings. 
  The result then follows, since a mixture of $(\epsilon, \vec{\delta})$-DP mechanisms is also 
  $(\epsilon, \vec{\delta})$-DP.
\end{proof}

\subsection{Converting zero-concentrated DP to approximate DP} \zlabel{app:non-uniform-conversion}

It is convenient to analyze the privacy of the interaction $\sinteract(\vec{X}; \analyst, \mech)$ for the clipped 
Gaussian mechanism using $\vec{\rho}$-zCDP, since it provides tighter accounting than $(\epsilon, \vec{\delta})$-DP. 
However, we ultimately need to convert to $(\epsilon, \vec{\delta})$-DP in order to obtain a generalization guarantee 
using \zcref{thm:transfer-dp}. 
\citet{canonne2020discrete} provide a conversion result from $\rho$-zCDP to $(\epsilon, \delta)$-DP when the privacy 
parameters are uniform. 
Here we generalize their result to non-uniform privacy parameters. 

We begin by generalizing Lemma~9 of \citet{canonne2020discrete}. 
This is a technical result that lower bounds $\vec{\delta}(i)$ in terms of an expectation involving the privacy loss 
random variable. 
\begin{lemma} \zlabel{lem:approx-dp-delta-vs-epsilon}
  Let $\epsilon \geq 0$ and $\vec{\delta} \colon \irange{n} \to [0, \infty)$. 
  A randomized mechanism $\mech \colon \scrX^\star \to \scrY$ satisfies $(\epsilon, \vec{\delta})$-DP if and only if 
  \begin{align}
    \vec{\delta}(i) 
      \geq \E_{Z_i \gets \privloss{\mech(\vec{x})}{\mech(\vec{x}')} } \mathopen{} 
        \left(\max \left\{0, 1 - \euler^{\epsilon - Z_i}\right\}\right)
  \end{align}
  for all indices $i \in \irange{n}$ and neighboring datasets $(\vec{x}, \vec{x}') \in \scrN_i$. 
\end{lemma}
\begin{proof}
  Fix $i \in \irange{n}$ and neighboring datasets $(\vec{x}, \vec{x}') \in \scrN_i$. 
  Let $Z_i \gets \privloss{\mech(\vec{x})}{\mech(\vec{x}')}$ and 
  $Z_i' \gets \privloss{\mech(\vec{x}')}{\mech(\vec{x})}$. 
  Our goal is to prove that
  \begin{equation}
    \sup_{E \subseteq \scrY} \Pr(\mech(\vec{x}) \in E) - \euler^{\epsilon} \Pr(\mech(\vec{x}') \in E)
    = \E(\max \{0, 1 - \euler^{\epsilon - Z_i}\}).
  \end{equation}
  
  For any $E \subseteq \scrY$, we have %
  \begin{equation}
    \Pr(\mech(\vec{x}') \in E)
    = \E \mathopen{} \left( \ind{\mech(\vec{x}') \in E} \right)
    = \E \mathopen{} \left( \ind{\mech(\vec{x}) \in E} \euler^{-Z_i} \right).
  \end{equation}
  Thus for all $E \subseteq \scrY$, we have
  \begin{align}
    \Pr(\mech(\vec{x}) \in E) - \euler^{\epsilon} \Pr(\mech(\vec{x}') \in E) 
    &= \E \mathopen{} \left(\ind{\mech(\vec{x}) \in E}\right)
      - \euler^{\epsilon} \E \mathopen{} \left( \ind{\mech(\vec{x}) \in E} \euler^{-Z_i} \right) \\
    &= \E \mathopen{} \left( \ind{\mech(\vec{x}) \in E} (1 - \euler^{\epsilon -Z_i}) \right).
  \end{align}
  The worst event is $E = \{y \in \scrY : 1 - \euler^{\epsilon - Z_i} > 0 \}$ .
  Thus 
  \begin{align}
    \sup_{E \subseteq \scrY} \Pr(\mech(\vec{x}) \in E) - \euler^{\epsilon} \Pr(\mech(\vec{x}') \in E) 
    &= \E \mathopen{} \left( \ind{1 - \euler^{\epsilon - Z_i} > 0} (1 - \euler^{\epsilon - Z_i}) \right) \\
    &= \E(\max \{0, 1 - \euler^{\epsilon - Z_i}\})
  \end{align}
  as required.
\end{proof}

We can then use this lemma to generalize Proposition~12 of \citet{canonne2020discrete}, which converts R\'enyi DP to 
approximate DP. 
This is a step towards our final goal of converting zCDP to approximate DP, since zCDP is equivalent to enforcing  
R\'enyi DP over a range of privacy parameters.
\begin{proposition} \zlabel{prop:rdp-to-approx-dp}
  Let $\mech \colon \scrX^\star \to \scrY$ be a randomized mechanism. 
  Let $\alpha \in (1, \infty)$ and $\epsilon \geq 0$. 
  Suppose $\renyi{\alpha}{\mech(\vec{x})}{\mech(\vec{x}')} \leq \vec{\rho}(i)$ for all indices $i$ and neighboring 
  datasets $(\vec{x}, \vec{x}') \in \scrN_i$, where $\renyi{\alpha}{P}{Q}$ is the R\'enyi divergence of order 
  $\alpha$ of distribution $P$ from distribution $Q$. 
  Then $\mech$ is $(\epsilon, \vec{\delta})$-differentially private for 
  \begin{equation*}
    \vec{\delta}(i) = \frac{\euler^{(\alpha - 1)(\vec{\rho}(i) - \epsilon)}}{\alpha - 1} \left( 1 - \frac{1}{\alpha} \right)^\alpha.
    \label{eqn:delta-rdp-to-approx-dp}
  \end{equation*}
\end{proposition}
\begin{proof}
  Fix index $i \in \irange{n}$, neighboring datasets $(\vec{x}, \vec{x}') \in \scrN_i$ and let 
  $Z_i \gets \privloss{\mech(\vec{x})}{\mech(\vec{x}')}$.
  By assumption we have 
  \begin{equation}
    \E \mathopen{} \left(\euler^{(\alpha - 1) Z_i}\right) = \euler^{ (\alpha - 1) \renyi{\alpha}{\mech(\vec{x})}{\mech(\vec{x}')} }
      \leq \euler^{(\alpha - 1) \vec{\rho}(i)}.
  \end{equation}
  By \zcref{lem:approx-dp-delta-vs-epsilon}, we seek an upper bound on $\E(\max\{0, 1 - \euler^{\epsilon - Z_i}\})$, 
  which we can set to $\vec{\delta}(i)$. 
  Following \citeauthor{canonne2020discrete}, we pick $c > 0$ such that 
  $\max\{0, 1 - \euler^{\epsilon - z}\} \leq c \cdot \euler^{(\alpha - 1) z}$ for all $z \in \reals$. 
  Then 
  \begin{equation}
    \E \mathopen{} \left(\max \{0, 1 - \euler^{\epsilon - Z_i} \}\right)
      \leq \E \mathopen{} \left(c \cdot \euler^{(\alpha - 1) Z_i}\right)
      \leq c \cdot \euler^{(\alpha - 1) \vec{\rho}(i)}.
  \end{equation}
  \citeauthor{canonne2020discrete} show that the smallest value of $c$ that satisfies this condition is  
  $c = \frac{\euler^{\epsilon (1 - \alpha)}}{\alpha - 1} \left( 1 - \frac{1}{\alpha} \right)^\alpha$.
  Thus 
  \begin{equation}
    \E \mathopen{} \left(\max \{0, 1 - \euler^{\epsilon - Z_i}\}\right)
      \leq \frac{\euler^{\epsilon (1 - \alpha)}}{\alpha - 1} \left( 1 - \frac{1}{\alpha} \right)^\alpha 
        \cdot \euler^{(\alpha - 1) \vec{\rho}(i)}
      = \frac{\euler^{(\alpha - 1)(\vec{\rho}(i) - \epsilon)}}{\alpha - 1} \left( 1 - \frac{1}{\alpha} \right)^\alpha
      = \vec{\delta}(i).
  \end{equation}
\end{proof} 

By exploiting the connection between R\'enyi DP and zCDP, we obtain a conversion result from zCDP to approximate DP. 
This is a generalization of Corollary~13 of \citet{canonne2020discrete}.
\begin{corollary} \zlabel{cor:zcdp-approx-dp-conversion}
  Let $\mech \colon \scrX^\star \to \scrY$ be a randomized mechanism that satisfies $\vec{\rho}$-zCDP. 
  Then $\mech$ is $(\epsilon, \vec{\delta})$-DP for any $\epsilon \geq 0$ and $\vec{\delta}: \irange{n} \to [0, \infty)$ 
  such that 
  \begin{equation}
    \vec{\delta}(i) = \frac{\euler^{(\alpha^\star - 1)(\alpha^\star \vec{\rho}(i) - \epsilon)}}
      {\alpha^\star - 1} \left( 1 - \frac{1}{\alpha^\star} \right)^{\alpha^\star},
  \end{equation}
  with $\alpha^\star = \arg \min_{\alpha \in (1, \infty)} 
  \frac{\euler^{(\alpha - 1)(\alpha \sup_{i} \vec{\rho}(i) - \epsilon)}}{\alpha - 1} \left( 1 - \frac{1}{\alpha} \right)^\alpha.$
\end{corollary}
\begin{proof}
  Fix $i \in \irange{n}$ and $(\vec{x}, \vec{x}') \in \scrN_i$. 
  Since $\mech$ satisfies $\vec{\rho}$-zCDP, we have 
  $\renyi{\alpha}{\mech(\vec{x})}{\mech(\vec{x}')} \leq \alpha \vec{\rho}(i)$ for all $\alpha \in (1, \infty)$. 
  \zcref{prop:rdp-to-approx-dp} (with $\vec{\rho}(i) \gets \alpha \vec{\rho}(i)$) provides a conversion result 
  to $(\epsilon, \vec{\delta})$-DP for any choice of $\alpha$. 
  We choose $\alpha$ to minimize the worst case $\vec{\delta}(i)$, given by:
  \begin{equation*}
    \sup_{i} \vec{\delta}(i) 
    = \frac{\euler^{(\alpha - 1)(\alpha \sup_i \vec{\rho}(i) - \epsilon)}}{\alpha - 1} \left( 1 - \frac{1}{\alpha} \right)^{\alpha}
    = \exp g(\alpha)
  \end{equation*}
  with $g(\alpha) = (\alpha - 1)(\alpha \sup_i \vec{\rho}(i) - \epsilon) + \alpha \log (1 - 1 / \alpha) - \log (\alpha - 1)$.
  A unique minimizer $\alpha^\star$ exists since $g(\alpha)$ is a smooth convex function.
\end{proof}

We note that the optimizer $\alpha^\star$ can be found efficiently using binary search as described by 
\citet{canonne2020discrete}.

\section{Additional Empirical Results} \zlabel{app:experiments}

In this appendix, we provide additional empirical results to complement those presented in 
\zcref{sec:application-empirical}.
In \zcref{fig:K-vs-E-alpha-0.1-fixed-growth-ratio} of the main paper, we plotted a lower bound on the number of 
adaptive statistical queries $k$ that can be answered as a function of the final dataset size $n$. 
There we varied the initial size of the dataset $n_0$ to ensure a fixed growth ratio $n / n_0 = 3$. 
In \zcref{fig:k-vs-n-alpha-0.1-fixed-n0}, we produce a similar plot where we fix $n_0 = 500\,000$ and vary the 
growth ratio $n / n_0$ instead. 
Here again, we observe that our bounds (\textbf{Ours-U} and \textbf{Ours-N}) outperform the others in the non-asymptotic 
regime plotted, with the performance gap becoming more pronounced for larger values of $b$.  

\zcref{fig:alpha-vs-n-3-growth-ratio} covers the same setting as \zcref{fig:K-vs-E-alpha-0.1-fixed-growth-ratio} 
in the main paper, except that it plots the confidence width $\alpha$ on the vertical axis for a fixed number of 
queries $k = 10\,000$. 
A smaller confidence width is better, and we see that the relative rankings of the bounds is the same as in 
\zcref{fig:K-vs-E-alpha-0.1-fixed-growth-ratio}. 
Notably, the behavior of our bounds (\textbf{Ours-U} and \textbf{Ours-N}) is stable for all values of $b > 1$, 
whereas the confidence width degrades for the static-based bounds (\textbf{JLNRSS} and \textbf{JLNRSS+}) as 
$b$ increases. 

\zcref{fig:k-vs-b-alpha-0.1-fixed-n} examines the impact of query batching for a fixed final dataset size 
$n = 1\,500\,000$ and fixed initial dataset size $n_0 = 500\,000$.
It shows that both of our bounds (\textbf{Ours-U} and \textbf{Ours-N}) improve as the number of batches $b$ increases, 
reaching a saturation point at around $b = 40$. 
This suggests it is better from a generalization perspective to ask queries more frequently in smaller batches when 
using our bounds.
In contrast, the static-based bounds (\textbf{JLNRSS} and \textbf{JLNRSS+}) degrade as the number of batches $b$ 
increases. 

\begin{figure}
  \centering
  \includegraphics[width=0.95\textwidth]{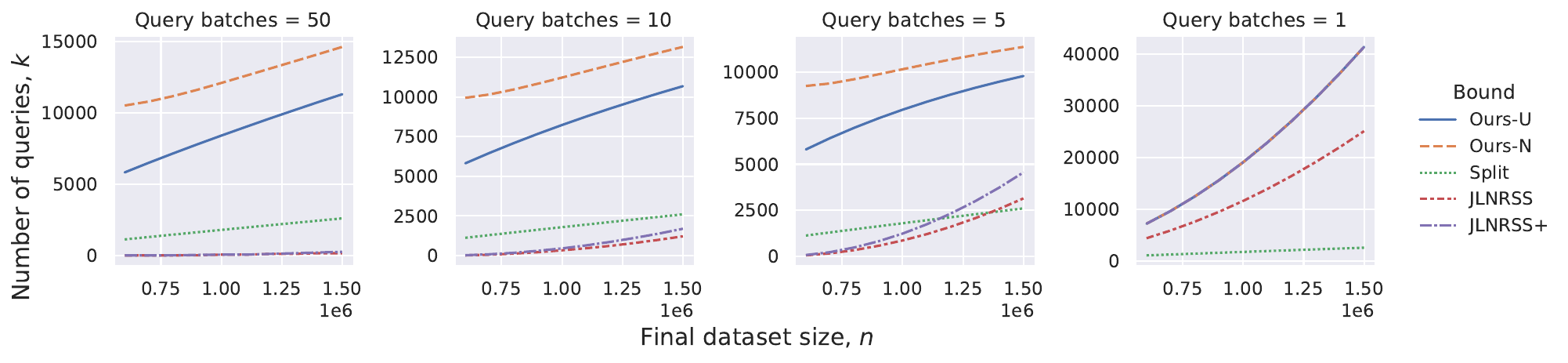}
  \caption{Comparison of the number of adaptive statistical queries that can be answered with error tolerance 
    $\alpha = 0.1$ and uniform coverage probability $1 - \beta = 0.95$ using a growing dataset with fixed
    initial size $n_0 = 500\,000$ in a batched query setting. 
    The number of queries (vertical axis) is plotted as a function of the final dataset size $n$ (horizontal axis), 
    bound (curve style) and the number of query batches $b$ (horizontal panel). 
    The right-most panel with $b = 1$ corresponds to the static data setting.}
  \zlabel{fig:k-vs-n-alpha-0.1-fixed-n0}
\end{figure}

\begin{figure}
  \centering
  \includegraphics[width=0.95\textwidth]{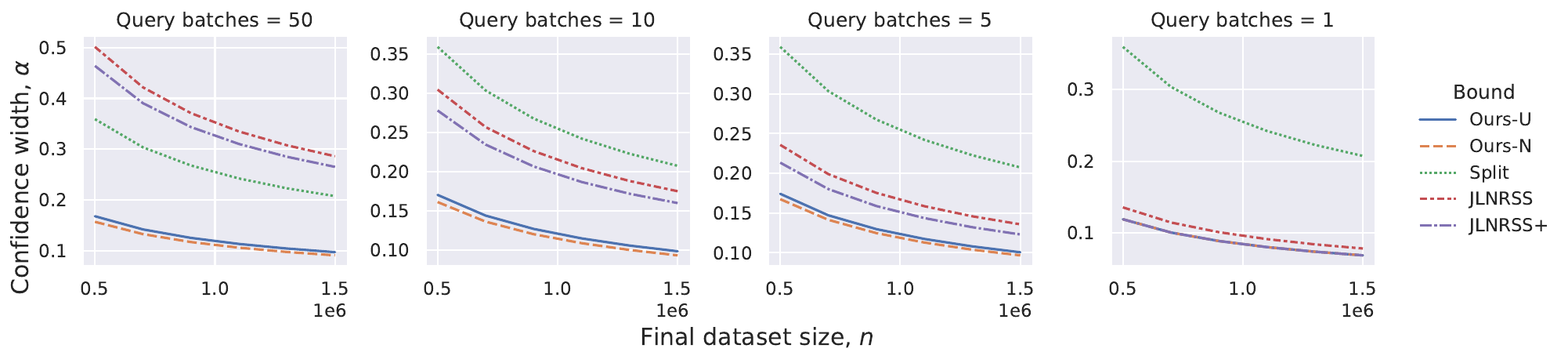}
  \caption{Comparison of the error tolerance $\alpha$ (vertical axis) for $k = 10\,000$ adaptive statistical queries 
    with uniform coverage probability $1 - \beta = 0.95$ using different ADA bounds (curve style). 
    The queries are answered using a growing dataset with final size $n$ (horizontal axis) and growth ratio 
    $n / n_0 = 3$, and are grouped into $b$ batches (horizontal panel). 
    The right-most panel with $b = 1$ corresponds to the static data setting.}
  \zlabel{fig:alpha-vs-n-3-growth-ratio}
\end{figure}

\begin{figure}
  \centering
  \includegraphics[width=0.7\linewidth]{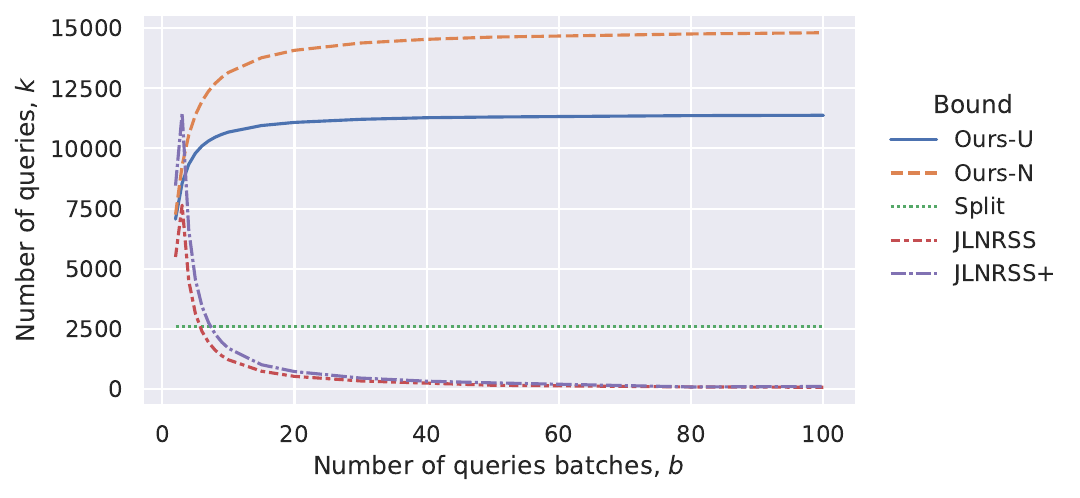}
  \caption{Comparison of the number of adaptive statistical queries that can be answered with error tolerance 
    $\alpha = 0.1$ and uniform coverage probability $1 - \beta = 0.95$ using a growing dataset with 
    initial size $n_0 = 500\,000$ and final size $n = 1\,500\,000$ in a batched query setting. 
    The number of queries (vertical axis) is plotted as a function of the number of query batches $b$ (horizontal 
    axis) and bound (curve style).}
  \zlabel{fig:k-vs-b-alpha-0.1-fixed-n}
\end{figure}

\end{document}